\documentclass[letterpaper]{article} 
\usepackage[preprint]{tmlr}

\usepackage{amsmath,amsfonts,bm}

\def\eqref#1{equation~\ref{#1}}

\def\1{\bm{1}}

\DeclareMathAlphabet{\mathsfit}{\encodingdefault}{\sfdefault}{m}{sl}
\SetMathAlphabet{\mathsfit}{bold}{\encodingdefault}{\sfdefault}{bx}{n}

\usepackage{algorithm}
\usepackage{algorithmic}

\usepackage{newfloat}
\usepackage{listings}
\usepackage{wrapfig}
\usepackage[table]{xcolor}
\usepackage{colortbl}
\usepackage{nicematrix}
\usepackage{booktabs}
\usepackage{subcaption}
\usepackage{multirow}
\usepackage{pdflscape}
\usepackage{longtable}
\usepackage{placeins}
\usepackage{capt-of} %

\floatstyle{ruled}
\newfloat{listing}{tb}{lst}{}
\floatname{listing}{Listing}

\usepackage{amsmath}      %
\usepackage{amssymb}      %

\usepackage{amsthm}

\newtheorem{theorem}{Theorem}%
\newtheorem{corollary}{Corollary}%

\theoremstyle{definition}

\theoremstyle{remark}
\newtheorem*{remark}{Remark}

\usepackage{graphicx}     %
\usepackage{subcaption}   %
\usepackage{enumitem}     %
\usepackage{mathtools}    %

\def\asterix{{\includegraphics[height=\fontcharht\font`M]{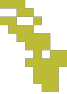}}}
\def\cauldron{{\includegraphics[height=\fontcharht\font`M]{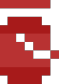}}}

\newcommand{\cheat}{\vspace{-2.4mm}}

\newcommand{\relcoords}{ORCA}
\newcommand{\relcoordslo}{Object-Relative Coordinates fixed to Anchor}
\newcommand{\relcoordslong}{\relcoordslo{} (\relcoords{})}
\newcommand{\Relcoordslong}{\MakeTitlecase[words=first]{\relcoordslong{}}}

\newcommand{\numgames}{13}

\newcommand{\game}[1]{\textit{#1}}
\newcommand{\framework}[1]{\textit{#1}}
\newcommand{\variant}[1]{\texttt{#1}}

\newcommand{\ppocnn}{PPO$^\text{\tiny CNN}$}
\newcommand{\ppomlp}{PPO$^\text{\tiny MLP}$}

\usepackage{hyperref}
\usepackage{url}
\usepackage[capitalize, noabbrev]{cleveref}
\Crefname{appsec}{Appendix}{Appendices}

\hypersetup{
    colorlinks=true,
    allcolors=black
}

\definecolor{shade}{gray}{0.94}

\title{Reframing Object-Centric Reinforcement Learning:\\
From Absolute to Relative Representations}

\author{%
\name Cedric Derstroff\thanks{These authors contributed equally.}
\email cedric.derstroff@tu-darmstadt.de \\
\addr The Hessian Center for Artificial Intelligence (hessian.AI) \\
Technische Universität Darmstadt \\\cheat{}
\AND
\name Jannis Blüml\footnotemark[1] 
\email jannis.blueml@tu-darmstadt.de \\
\addr The Hessian Center for Artificial Intelligence (hessian.AI)\\
Technische Universität Darmstadt\\\cheat{}
\AND
\name Janosch Moos 
\email janosch.moos@tu-darmstadt.de \\
\addr Technische Universität Darmstadt \\\cheat{}
\AND
\name Michael Girstl 
\email michael.girstl@tu-darmstadt.de \\
\addr The Hessian Center for Artificial Intelligence (hessian.AI)\\
Technische Universität Darmstadt\\\cheat{}
\AND
\name Zihan Ye 
\email zihan.ye@tu-darmstadt.de \\
\addr The Hessian Center for Artificial Intelligence (hessian.AI)\\
Technische Universität Darmstadt\\\cheat{}
\AND
\name Debora Clever
\email debora.clever@tu-darmstadt.de \\
\addr Technische Universität Darmstadt\\\cheat{}
\AND
\name Kristian Kersting 
\email kristian.kersting@tu-darmstadt.de \\
\addr The Hessian Center for Artificial Intelligence (hessian.AI) \\
Centre for Cognitive Science \\
German Research Center for Artificial Intelligence (DFKI) \\
Technische Universität Darmstadt \\
}

\def\month{MM}  %
\def\year{YYYY} %
\def\openreview{\url{https://openreview.net/forum?id=XXXX}} %

\begin{document}

\maketitle  %

\vspace{-3mm}

\begin{abstract}
Object-centric reinforcement learning agents typically receive each object’s position in a fixed global frame, even though their decisions usually depend on where objects are relative to one another. These absolute views often fail to capture the local context that is most relevant to the agent’s behavior. As a result, an agent must relearn the same response at every location where a configuration can occur, and its performance collapses when the scene is merely shifted.
In this work, we propose a shift toward Object-Relative Coordinates fixed
to Anchor (ORCA), where entities are described in a relative frame of reference. To operationalize this idea, we introduce Object-Relative inductive-Bias Transformers (ORBiT), a transformer-based architecture for processing anchor-relative inputs in reinforcement learning.
Our results suggest these relative representations improve performance for both MLP and transformer agents. Furthermore, we demonstrate that ORBiT outperforms the pixel-based PPO baseline, achieving an aggregate HNS IQM greater than 1.0 over 13 Atari games. ORBiT learns the game of Foosball from scratch, dominating all baselines. In our controlled Tag environment, we further show that ORCA extends to rotations, that ORBiT generalizes to unseen numbers of distractors, and that these gains require a decision-relevant anchor.
\end{abstract}

\section{Introduction}

Ask a person where their keys are, and they will not answer in coordinates; they will say \emph{next to the bowl}, \emph{on the left side of the desk}, or \emph{behind you}.
Human spatial reasoning often uses spatial anchors and relative frames of reference~\citep{levinson2003space}. Scenes are parsed into bounded objects with persistent identity, and the position of each object is encoded in a frame tied to the body or to nearby landmarks rather than in a single fixed global system~\citep{klatzky1998allocentric,spelke1990principles, scholl2001objects, burgess2006spatial}.
Object-centric reinforcement learning (OCRL) has internalized the first of these two ideas: rather than feeding pixels to a policy, OCRL extracts a set of objects with type, position, and velocity, and lets the agent reason over the resulting structured representation~\citep{diuk2008object,Baker2020Emergent,delfosse2024ocatari,SCOBOTS,ye2025learning,ye2026boosting}.
The second idea, using anchors and relative frames of reference, has been independently adopted in multi-agent reinforcement learning, robotics, biology, and linguistics, where observations relative to the agent's position are more commonly used~\citep{gupta2017cognitive,wang2020egocentric,shang2021agent,levinson2003space}, yet it remains largely absent from OCRL.
The absolute-coordinate convention appears to have been inherited from early object-vector pipelines and rarely revisited.

\begin{figure}[t]
    \centering
    \includegraphics[width=\linewidth]{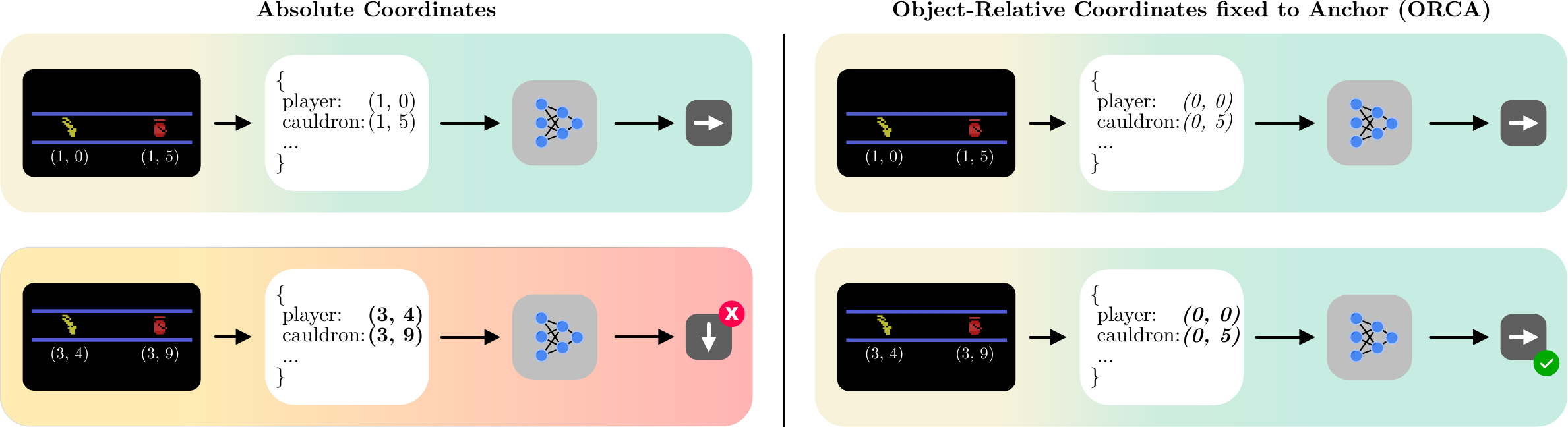}
    \caption{\textbf{With \relcoords{} inputs, the agent only needs to learn strategies for similar situations once, enabling direct generalization.} When the player (Asterix\;\asterix{}) is in the same row as the target collectible (cauldron\;\cauldron{}), e.g., the cauldron $(1, 5)$ is right of Asterix $(1,0)$, the absolute coordinates of the objects should not influence the agent's decision. The agent should behave the same even if the cauldron is at $(3,9)$, which is also on the right of Asterix $(3,4)$, but on a different platform. \Relcoordslong{} performs this transformation in object space, simplifying the agent's decision.}
    \label{fig:teaser}
\end{figure}

From classical object-vector pipelines using multi-layer perceptrons (MLPs) to recent transformer-based variants, OCRL serves object features to the policy in the \emph{absolute frame of reference or global coordinates} \citep{anand2019unsupervised,zadaianchuk2020selfsupervisedvisualreinforcementlearning, mambelli2022compositionalmultiobjectreinforcementlearning, haramatientity,delfosse2024ocatari}.
This leaves inter-object relations implicit: For example, a cauldron 5 pixels to the right of the player in \game{Asterix} maps to a different input at every screen location, so the agent must relearn the same response from scratch at each position (see \cref{fig:teaser}).
The result is fragility under spatial perturbations, frame translation, and platform shifts, even when the relative scene is unchanged.

We close this gap with \textbf{ORBiT} (Object-Relative inductive-Bias Transformer): an \emph{\relcoordslong{}} input that encodes each object's position and velocity in the relative frame of a designated anchor object (typically the player or important entities like the ball in \game{Foosball}), fed into a transformer encoder that is order-invariant and count-flexible by construction.
When trained on the \framework{Arcade Learning Environment (ALE)}~\citep{bellemare13arcade,machado18arcade} and tested across the \framework{HackAtari} benchmark~\citep{hackatari}, switching the inputs from absolute coordinates to \relcoords{} raises the interquartile mean (IQM) performance from $0.54$ to $1.10$ in the human-normalized score (HNS, see \cref{app:metrics}), moving it from below a pixel baseline without access to objects ($0.69$) to approximately human level.
Testing robustness, we see that ORBiT loses only $4\%$ of its score on the \framework{HackAtari} test suite, whereas absolute coordinates yield a $30\%$ drop.
In \game{Foosball} --- a challenging 8-DoF (degrees of freedom) robotics environment that requires robust policies in the presence of changing opponent behavior \citep{foosball} --- ORBiT concedes only 73 goals out of 50,000 against its absolute-coordinate baseline OCT. In \game{Tag}, an environment where the player has to catch a target without being distracted by other moving objects,
ORBiT reaches 22.4 catches within 500 steps, whereas OCT is barely able to catch the target (1.8).

\paragraph{Contributions.}
\begin{itemize}[leftmargin=*, topsep=0pt, itemsep=2pt, parsep=0pt]
    \item We formalize \textbf{\relcoordslong{}} representations, in which object positions and velocities are expressed relative to a task-relevant object.
    \item We present \textbf{ORBiT}, a transformer encoder over \relcoords{}-encoded object sets that is order-invariant and count-flexible by construction.
    \item We show that ORBiT improves end performance and robustness to spatial perturbations (translation and platform shifts) compared to absolute-coordinate variants, which substantially degrade under the same perturbations.
    \item We introduce \game{Tag}, a controlled environment with oriented, moving objects in which object count, the anchor, and landmarks can be varied independently. This lets us evaluate ORCA with an oriented anchor, where it applies the full rigid-body transform.
    \item We demonstrate on \game{Foosball} that ORBiT transfers to continuous-control tasks and outperforms both OCT (the same transformer with absolute coordinates) and classical MLP models (\ppomlp{}).
\end{itemize}

In the remainder of the paper, we formalize \relcoords{} representations, describe the ORBiT architecture, and evaluate its effectiveness across multiple environments and generalization settings.

\section{From Absolute to Relative Representations}  
\label{sec:aor}

\begin{wrapfigure}[22]{r}{.52\linewidth}
    \vspace{-9pt}
    \centering
    \begin{subfigure}[t]{0.4\linewidth}
        \centering
        \includegraphics[width=\linewidth]{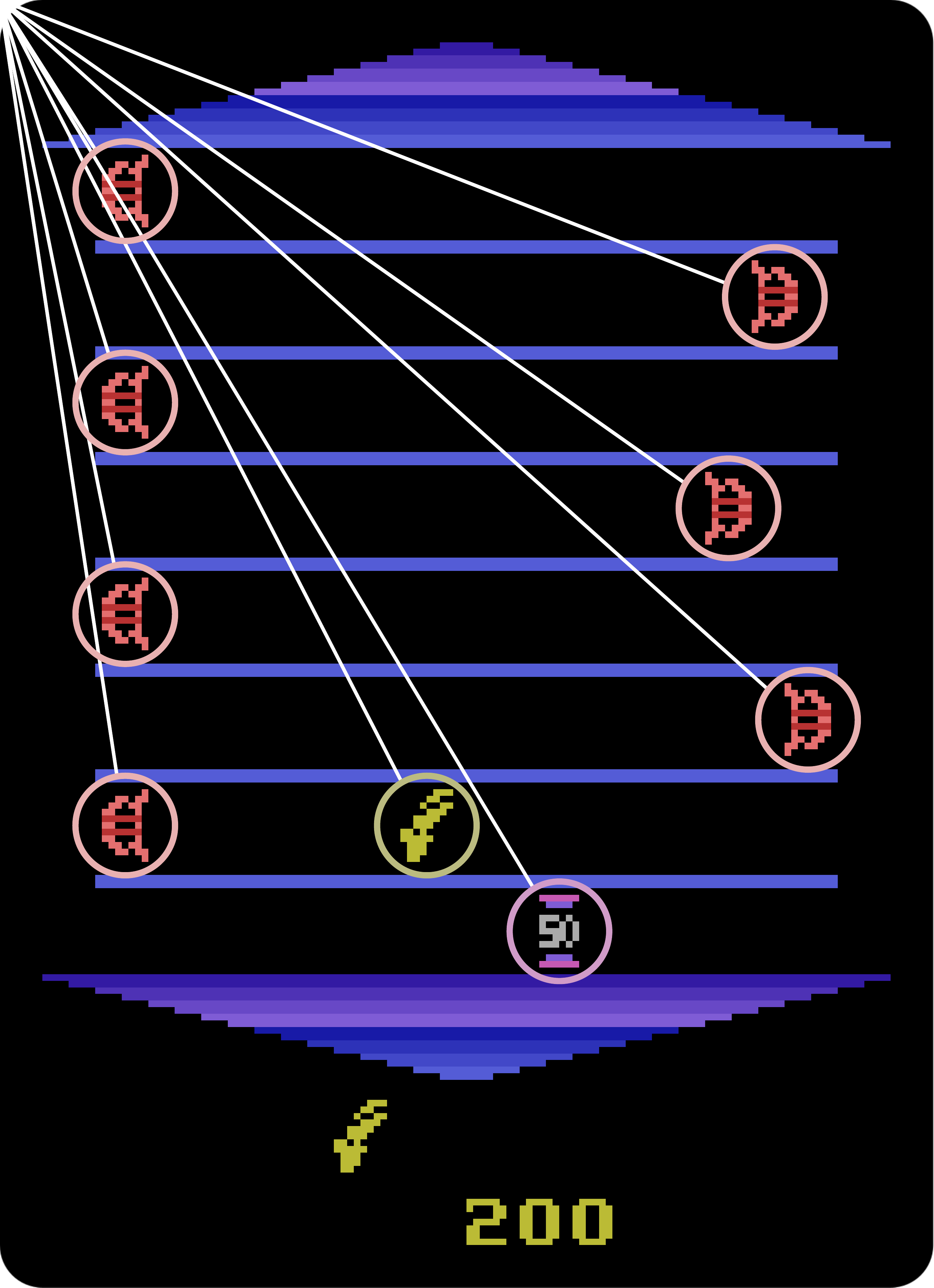}
        \caption{Absolute}
    \end{subfigure}%
    ~\hspace{2.5mm}
    \begin{subfigure}[t]{0.5\linewidth}
        \centering
        \includegraphics[width=0.8\linewidth]{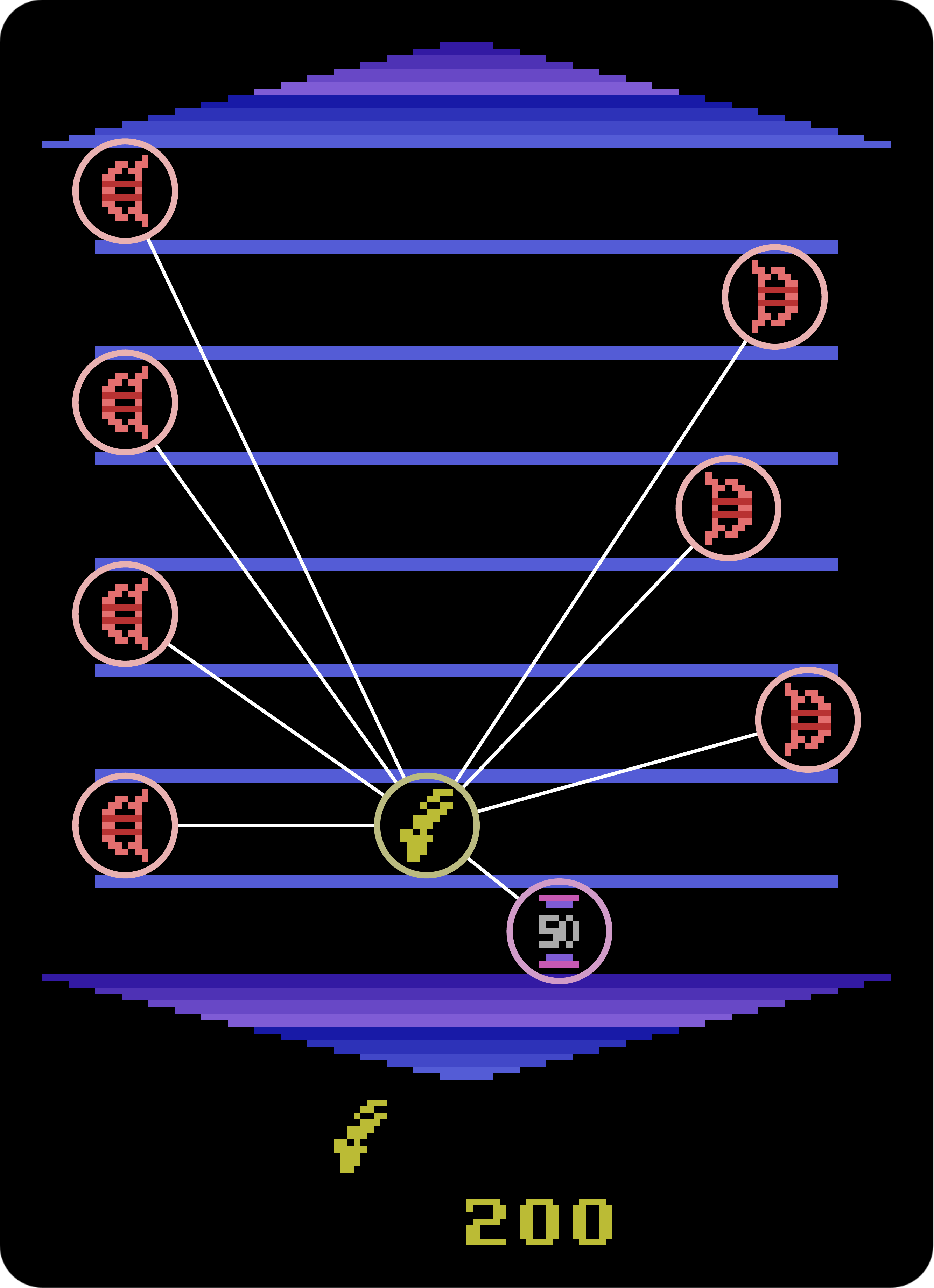}
        \caption{\centering Relative (\relcoords{})}
    \end{subfigure}
    \caption{\textbf{Comparison of absolute and relative representations in \game{Asterix}.} \textit{Left:} Absolute representations encode object positions with respect to a fixed origin in the environment (e.g., top-left corner). \textit{Right:} Relative representations encode object positions with respect to the anchor object (e.g., the agent’s location).}
    \label{fig:coordinate_systems}
\end{wrapfigure}

We model the environment with object-centric representations, where  the agent receives at each step $t$ a set of $N_t$ objects $\mathcal{O}_t = \{o^1_t, \dots, o^{N_t}_t\}$ from an upstream detector, such as SAM3~\citep{SAM3}, OCAtari~\citep{delfosse2024ocatari}, or similar.
Each object $o^i_t$ is a feature vector of \emph{intrinsic} $\mathbf{c}^i_t$ and \emph{extrinsic} $\mathbf{e}^i_t$ attributes.
Intrinsic attributes may cover the object type, color, and size, while the extrinsic attributes feature the absolute position $\mathbf{p}^i_t$, velocity $\mathbf{v}^i_t$, acceleration $\mathbf{a}^i_t$, and the heading angle $\varphi^i_t$ of the object in the absolute environment frame.
In our experiments, we limit $\mathbf{c}^i_t$ to a one-hot encoding of the object type and
$\mathbf{e}^i_t$ to $\mathbf{p}^i_t$ and $\mathbf{v}^i_t$, where we use $\mathbf{v}^i_t = \mathbf{p}^i_t - \mathbf{p}^i_{t-1}$ for Atari
(for \game{Foosball}, see \cref{app:foosball}).
OCRL
typically passes $\mathcal{O}_t$ directly to the policy as global values within an absolute frame of reference.

We instead express each object's position and velocity relative to a designated \emph{anchor} object $\psi_t$.
The anchor should represent the most important object in the game. In \game{Atari}, we typically use the player object --- e.g., the paddle in \game{Pong} or Asterix in \game{Asterix} --- as the relative anchor, whereas in \game{Foosball}, we use the ball.

Concretely, rather than encoding ``the cauldron at pixel $(3, 9)$'' in Asterix, we encode ``the cauldron 5 pixels right and 0 pixels above the player'' (see \cref{fig:coordinate_systems}).
Formally, let $\psi_t$ have position $\mathbf{p}^\psi_t$, velocity $\mathbf{v}^\psi_t$, and orientation $\mathbf{Q}_t \in \mathit{SO}(d)$; in $d=2$ dimensions, $\mathbf{Q}_t = \mathbf{R}(\varphi_t)$ for a heading angle $\varphi_t$, where $\mathbf{R}(\varphi_t)$ denotes the 2d rotation matrix parameterized by $\varphi_t$.
The anchor-relative observation
\begin{equation}
    \tilde{\mathcal{O}}_t = \{\tilde{o}^i_{\psi, t}\}_{i=1}^{N_t}
\end{equation}
is defined by the transformation
\begin{equation}
    \tilde{o}^i_{\psi, t} = T(\psi_t,\, o^i_t)
    = \big[\,\boldsymbol{\phi}_{\mathrm{int}}(o^i_t),\;\boldsymbol{\phi}_{\mathrm{rel}}(\psi_t,\, o^i_t)\,\big],
    \label{eq:aor}
\end{equation}
where $\boldsymbol{\phi}_{\mathrm{int}}$ returns $\mathbf{c}^i_t$ unchanged and $\boldsymbol{\phi}_{\mathrm{rel}}$ maps positions and velocities into the anchor's frame. 
With $\mathbf{R}_t := \mathbf{Q}_t^\top = \mathbf{R}(-\varphi_t)$ in $d = 2$,
\begin{equation}
    \boldsymbol{\phi}_{\mathrm{rel}}(\psi_t,\, o^i_t)
    = \big[\,
        \underbrace{\mathbf{R}_t \cdot (\mathbf{p}^i_t - \mathbf{p}^\psi_t}_{\Delta\mathbf{p}^i_t}),\;
        \mathbf{R}_t \cdot (\mathbf{v}^i_t - \mathbf{v}^\psi_t)
      \,\big].
    \label{eq:phi-rel}
\end{equation}
For \framework{Atari} and \framework{Foosball}, where the player and ball have no meaningful orientation, we assume $\mathbf{R}_t = \mathbf{I}$ and \cref{eq:phi-rel} reduces to subtracting the anchor's position and velocity from each object. Every Atari result we report is, therefore, evidence for translation invariance only; see \cref{rem:translation-only}. For objects with a heading $\varphi^i_t$, the rotation matrix $\mathbf{R}(-\varphi_t)$ encodes the frame from the perspective of the anchor object itself, as we can see in \game{Tag}. \game{Tag} also uses an extended version of  \cref{eq:phi-rel} (see \cref{app:tag}). The anchor object is self-encoded as $T(\psi_t, \psi_t)$, giving $\Delta\mathbf{p}^\psi_t = \mathbf{0}$ and zero relative velocity; this keeps the anchor object as a regular member of the object set (see \cref{fig:coordinate_systems}). \relcoords{} however discards the anchor's absolute pose and velocity. As we will see in the results, this does not pose a problem in most environments, since these environments contain static objects that can serve as landmarks. 
Furthermore, our method assumes that the anchor object exists at all times and can always be localized. RQ5 measures the extent to which the anchor choice matters.

\paragraph{Invariance.}
A practical consequence is that the policy does not need to relearn responses across spatial positions: two scenes that differ only by a global shift or rotation produce identical \relcoords{} observations.
Formally, $\tilde{\mathcal{O}}_t$ is invariant to any rigid transformation of the world frame applied jointly to all objects,  so a policy acting on $\tilde{O}_t$ returns the same action distribution for all scenes related by such a transformation, with no auxiliary losses.
A formal proof over $\mathit{SE}(d)$ and the orientation-free specialization used in our Atari experiments are in \cref{app:invariance_proof}. Where the anchor object carries a meaningful orientation, and actions are expressed in the world frame, the policy's output must be mapped back through $\mathbf{Q}_t$, making the composite policy equivariant rather than invariant.

\paragraph{Example.}

In \game{Asterix}, the same configuration --- a lyre or cauldron on the left of Asterix --- maps to a different absolute input on every platform; however, under \relcoords{}, it is always the same relative vector, so the policy only has to learn the response once (see \cref{fig:teaser,fig:coordinate_systems}). Similarly, this holds for the cars relative to the player (chicken) in \game{Freeway}.

\relcoords{} adds no learnable parameters and wraps any object-vector environment without modifying the downstream policy class. As such, \relcoords{} can be used with any object-centric approach with an identifiable potential anchor object. This includes methods such as an object-centric PPO (PPObj) and an absolute-coordinate transformer (OCT). 

\paragraph{Scene Graphs and GNNs.}Representing the scene as a graph over all objects is a genuine alternative, with a long history in relational RL, object-oriented MDPs, and graph networks~\citep{battaglia2018relational}; see \cref{sec:related_work}. 
Given displacement edge features $p^i_t - p^j_t$, such a model obtains the same translation invariance that \relcoords{} does, so the choice between the two is not about whether the invariance is available, but about what it costs to obtain.
A graph-based policy constructs edge features that scale quadratically with $N_t$, evaluates a learned edge function on each, and introduces design decisions, such as edge parameterization or aggregation schemes, that must be specified and tuned before the invariance can be of any use.
\relcoords{} changes the input representation while keeping the rest untouched: the policy, training algorithm, and reward stay as they are.
The relational content is preserved rather than traded away, as for any two objects $i$ and $j$,
\begin{equation}
  \mathbf{R}_t \cdot \bigl(\mathbf{p}^i_t - \mathbf{p}^\psi_t\bigr) - \mathbf{R}_t \cdot \bigl(\mathbf{p}^j_t - \mathbf{p}^\psi_t\bigr) \;=\; \mathbf{R}_t \cdot \bigl(\mathbf{p}^i_t - \mathbf{p}^j_t\bigr),
  \label{eq:pairwise}
\end{equation}
so every pairwise displacement remains recoverable as the difference between two anchor-relative tokens, and self-attention can form whichever of them the task requires.
This keeps the relational information that an explicit graph would provide while leaving the pipeline as simple as possible.
The cost is the anchor itself: a graph over all objects needs no designated object, whereas \relcoords{} requires one that exists and can be identified at every step.
We do not evaluate an explicit scene-graph variant and leave that comparison to future work.

\subsection{Anchor Choice}
Forgoing a full graph evaluation requires selecting a single anchor object. This choice mirrors real-world human spatial cognition, which is fundamentally biased toward an egocentric frame of reference \citep{wang2002human}. When playing video games, humans naturally adopt a player-centric perspective, which motivated our design of the coordinate system. We therefore use the player object as the anchor wherever the game provides one. 
More generally, we select as anchor the object whose relations to the rest of the scene are most decision-relevant, i.e., the object that the agent's actions most directly act upon or through.

However, a player-centric anchor is not the only possible choice. In \game{Foosball}, the ideal spatial reference shifts. Unlike the \framework{Atari} environments, \game{Foosball} features eleven distinct player objects under the agent's control. Rather than anchoring to any single player, the most critical element in the game is the ball. Consequently, we use the ball as our anchor for the game of \game{Foosball}.
We further use \game{Tag} to systematically investigate the influence of different anchor choices in ORBiT (see \cref{sec:experiments}). This illustrates that the anchor is best understood as a task-dependent design parameter rather than a fixed convention; automatically identifying the anchor, or combining several anchors, is a promising direction for future work.

\subsection{Object-Relative inductive-Bias Transformers (ORBiT)}
\label{sec:transformer}
\begin{figure}
    \centering
    \includegraphics[width=0.95\linewidth]{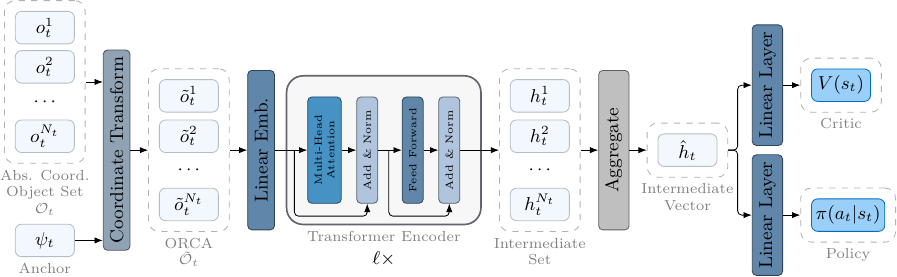}
    \caption{\textbf{ORBiT architecture: a versatile transformer with \relcoords{} inputs}. The absolute coordinate object set $\mathcal{O}_t$ is transformed into the anchor-relative set $\tilde{\mathcal{O}}_t$.
    These \relcoords{} inputs are embedded into the working dimension of the transformer. The transformer's output sequence is aggregated (mean-pooled) into a single vector $\hat{h}_t$ before being passed through the linear critic and policy heads. OCT uses $\mathcal{O}_t$ directly within the otherwise identical architecture.}
    \label{fig:orbit}
\end{figure}

To process object-centric observations flexibly and expressively, we build on the transformer architecture~\citep{vaswani2017attention}, which has proven effective at modeling interactions among sets of entities. The transformer’s raw self-attention mechanism enables reasoning over unordered inputs \citep{vaswani2017attention}, making it a natural fit for RL agents operating on object-based sets extracted from structured environments. ORBiT builds such a transformer over the \relcoords{}-encoded object set $\tilde{\mathcal{O}}_t$ from \cref{sec:aor}.

Each token $\tilde{o}^i_{\psi, t} \in \tilde{\mathcal{O}}_t$ concatenates intrinsic attributes $\mathbf{c}^i_t$ and relative spatial features $\boldsymbol{\phi}_{\mathrm{rel}}(\psi_t, o^i_t)$.
Object types are one-hot encoded, and the type vocabulary is fixed at initialization.
A shared linear layer projects each token to a $d$-dimensional embedding, giving $N_t$ embeddings as input to the transformer encoder layers. After $\ell$ encoder layers, the resulting set is aggregated into a single vector before being passed into the linear value and policy head (see \cref{fig:orbit}). More details in \cref{app:transformer}.

Our baseline OCT uses the same architecture and hyperparameters as ORBiT but takes absolute coordinates as input rather than \relcoords{}.
\section{Empirical Evaluation}
\label{sec:experiments}

We evaluate whether \relcoords{} representations and the ORBiT architecture (\cref{sec:transformer}) improve reinforcement learning performance and generalization. Our empirical analysis is organized around the following research questions:

\begin{enumerate}[label=\textbf{RQ\arabic*.}, leftmargin=*, itemsep=2pt]
    \item Does \relcoords{} outperform absolute-coordinate object inputs in OCRL?
    \item Is \relcoords{} architecture-agnostic?
    \item Are \relcoords{}-based agents more robust to test-time spatial perturbations (translation, frame shifts, platform repositioning) than absolute-coordinate agents?
    \item Does ORBiT transfer to continuous-control tasks beyond Atari?
    \item How much do ORCA's gains depend on the choice of anchor?
\end{enumerate}

\paragraph{Environments and Setup.}

We train all agents on \numgames{} \framework{Atari} games and evaluate robustness zero-shot using the \framework{HackAtari} perturbation suite~\citep{hackatari}, as well as random translations as proposed by \citet{rad}; perturbations are never seen during training.
Object states are extracted with \framework{OCAtari}'s RAM extraction method~\citep{delfosse2024ocatari} at each step.
We use IQM for aggregated comparisons to reduce sensitivity to outliers~\citep{agarwal2021deep} and report it together with the 95\% confidence interval (CI) (see \cref{app:metrics}). Retained performance is described as 
\begin{equation}
\varrho =  \frac{r_{var} - r_{min}}{r_{org} - r_{min}} 
\end{equation}
and measures the robustness to close variations of the game, with $r_{org}$ being the return after an episode in the original game, $r_{var}$ in the respective variant, and $r_{min}$ the minimal return possible. For our standard \framework{Atari} experiments, we instead use the human-normalized score (HNS; see \cref{app:metrics} and \citet{BadiaPKSVGB20}) to enable aggregation across games.
Five agents are compared throughout:
\textbf{\ppocnn{}} (pixel-based CNN, no object extraction),
\textbf{PPObj} (MLP over absolute-coordinate object vectors, no transformer),
\textbf{PPObj-\relcoords{}} (same architecture as PPObj, but with \relcoords{} feature representation),
\textbf{OCT} (transformer over absolute-coordinate objects, same architecture as ORBiT),
and \textbf{ORBiT} (transformer over \relcoords{}-encoded objects; ours).
PPObj and OCT together independently isolate the contributions of the transformer and the \relcoords{} representation.
All agents are trained on 40 million frames or 10 million time steps across 5 seeds; the implementation follows \framework{CleanRL}~\citep{huang2022cleanrl} for PPO-like training and is extended for OCT and ORBiT.
Each run takes approximately 4 hours and 15 minutes on average on a single NVIDIA V100 GPU. These experiments are the basis for addressing research questions RQ1--RQ3.

For RQ4, ORBiT is evaluated against OCT and \ppomlp{} (absolute coordinates and joint values, classical MLP) on a \game{Foosball} environment derived from \citet{foosball}. 
Each agent is trained for roughly 6.6 billion time steps on 4,096 parallel environments. 

For  RQ5, we additionally use \game{Tag} (see \cref{app:tag}), a controlled 2D environment. In a bounded arena, the agent must catch the orange target, while purple targets act as distractors. Unlike \game{Atari} and \game{Foosball}, objects in \game{Tag} carry a heading, so \relcoords{} applies the full rigid-body transform (translation and rotation) rather than translation alone. Tag also lets us vary the number of objects and the choice of anchor. Starting from a base configuration (5 distractors, the player as anchor), we vary one factor at a time. We train OCT, ORBiT, PPObj, and PPObj-ORCA for 3 million time steps across 5 seeds.

Full hyperparameter details, as well as a complete list of the used Atari games, are in \crefrange{app:hyperparams}{app:foosball}.

\subsection{Evaluating the Performance}
\paragraph{RQ1 --- Relative coordinate systems heavily improve object-centric methods.}

\Cref{fig:aggregate} shows that
player-relative \relcoords{} inputs outperform global coordinates.
The transformer agent gains $\Delta{+}0.56$ HNS (from $0.54$ to $1.10$) with 95 \% CI [$0.49$, $0.63$], reaching human level.

\begin{wrapfigure}[26]{r}{0.53\linewidth}
    \vspace{-0.33cm}
    \includegraphics[width=\linewidth]{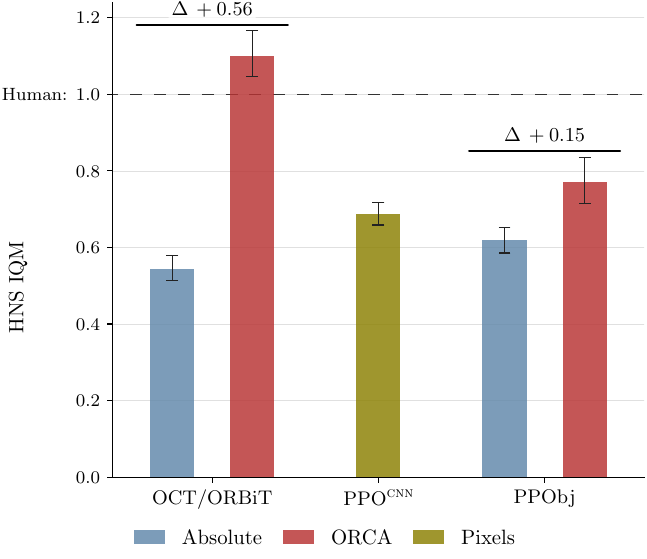} 
    \caption{\textbf{\relcoords{} improves both architectures in Atari.} The gain in HNS IQM is larger for the transformer (OCT/ORBiT, $\Delta{+}0.56$) than for the MLP (PPObj, $\Delta{+}0.15$).
    Both absolute-coordinate models fall below the pixel baseline (\ppocnn{}); only ORBiT approaches human level (dotted). Aggregated end performance (HNS IQM) with 95\% confidence intervals. Higher is better. Full per-game numbers in \cref{app:ext}.}
    \label{fig:aggregate}
\end{wrapfigure}

While the absolute coordinate transformer OCT falls behind the pixel-input \ppocnn{} (IQM: $0.69$, 95\% CI: [$0.66$, $0.72$]), despite receiving clean object states rather than raw pixels, ORBiT gains enough performance through \relcoords{} to clear this baseline. This means that without a reference anchor, the transformer fails to capture position-dependent correlations that do not hold across the game's spatial extent.
The confidence interval for ORBiT (OCT+\relcoords{}) spans $1.04$--$1.16$ IQM and still shows a substantial improvement over pure OCT. The same holds in the \game{Tag} environment, where ORBiT reaches 22.4 catches, compared with 1.8 for OCT (see \cref{fig:tag}).

\paragraph{RQ2 — \relcoords{} benefits both MLP and transformer agents.}
To separate these contributions, we isolate two effects: first, the improvement using \relcoords{}; second, we measure the difference between OCT and PPObj, isolating the architecture effect (both use global inputs). Here, the transformer alone does not help.

Besides the benefits of using \relcoords{} within the transformer agent, we can also observe an improvement in performance of $\Delta{+}0.15$ (from $0.62$ to $0.77$) with 95\% CI [$0.09$, $0.22$] within the MLP agent PPObj to bring it in line with the \ppocnn{} baseline (IQM: $0.69$, 95\% CI: [$0.66$, $0.72$]).

However, neither improvement alone clears the human baseline; PPObj-\relcoords{} sits at $0.77$ (95\% CI [$0.71$, $0.84$]) and OCT at $0.54$ (95\% CI [$0.51$, $0.58$]), but the combination does, reaching $1.10$ (see \cref{fig:aggregate}) with a 95\% CI [$1.04$, $1.16$].

The most plausible explanation is that the transformer can attend to each object's anchor-relative features independently and compose them, whereas an MLP concatenates all objects into a fixed-size vector and must re-learn any relational structure the \relcoords{} encoding makes explicit. Overall, we can say that \relcoords{} improved performance in both cases.
The substantial gap between $\Delta{+}0.56$ and $\Delta{+}0.15$ is primarily driven by a few games, e.g., \game{Asterix} (see \cref{tab:ext_atari}), suggesting that the transformer may disproportionately amplify \relcoords{} in specific games while maintaining a similar impact across most others.

The benefits of \relcoords{} are consistent across both architectures we test. 
This implies that relative coordinates should be used instead of absolute coordinates whenever the task structure allows.

\subsection{Robustness and Generalization}
\begin{figure}[t]
    \centering
    \includegraphics[width=\linewidth]{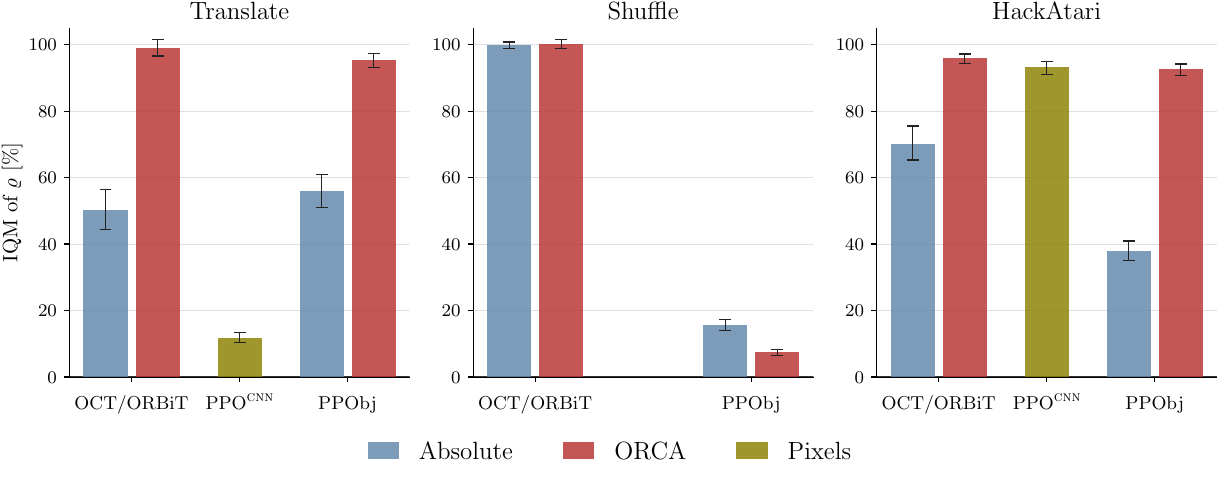}
    \caption{\textbf{ORBiT enhances robustness and retains the highest performance in most settings.} The transformer alone explains the full performance in OCT and ORBiT for \variant{shuffle}, which is not applicable for the CNN-based \ppocnn{} that operates on image data. In \framework{HackAtari}, ORBiT retains its performance similar to \ppocnn{}, and \variant{translate} tests \cref{cor:policy-equivariance}.
    The IQM of the \textit{retained performance} $\varrho$ when evaluating zero-shot under \variant{translated} inputs (left, shifts the absolute reference frame while keeping relative structure) and \variant{shuffle} (middle, permutes the object list while preserving the scene). Bars compare absolute coordinates and \textit{player-centric} (\relcoords{}) representations across architectures: \textbf{OCT} (transformer over objects), \textbf{\ppocnn{}} (pixel CNN; no relative variant), and \textbf{PPObj} (MLP over objects). Higher is better. IQM over $5 \cdot 10$ episodes.}
\label{fig:robustness}
    \label{fig:drop}
\end{figure}

\paragraph{RQ3 --- ORBiT is more robust to perturbations.}
We evaluate agents' zero-shot performance on three types of perturbations and report the retained performance $\varrho$ relative to clean performance. Results can be seen in \cref{fig:robustness}.
\variant{Translate} trims a random portion of the empty border, shifting the game content without cutting any of it and keeping the relative scene unchanged. For object-based agents, the same pixel offset is added to all extracted coordinates.
\variant{Shuffle} permutes the object set indices: the scene is identical, but any reliance on slot order breaks.
Note that \variant{Translate} and \variant{HackAtari} variants are only available on a set of 10 of 13 games, though not the same 10; details and full per-game numbers in \cref{app:ext}. Examples are in \cref{app:perturbations}.

As shown in \cref{thm:invariance} (\cref{app:invariance_proof}), \relcoords{} is guaranteed to be invariant to transformations in the coordinate frame. Hence, these results on \variant{Translate} serve as an empirical check of the follow-up \cref{cor:policy-equivariance}.
OCT loses $\Delta = -50\%$ under global and is essentially unaffected ($\approx 1\%$) under player-centric (ORBiT).
PPObj drops 44\% under global inputs and is also essentially unaffected ($5\%$) with \relcoords{}. Since \relcoords{} inputs are identical by construction, deviations from 100\% reflect evaluation stochasticity.
\ppocnn{} collapses by 88\%, with no anchor-relative variant to compare.
Looking at \variant{shuffle}, OCT is unaffected regardless of representation ($\Delta \approx 0$) --- permutation-invariance holds at test time and does not depend on \relcoords{} being active.
PPObj drops 85--93\% regardless of representation; its MLP has no analogous bias, and \relcoords{} does not compensate.

In the \framework{HackAtari} results, we observe that the global coordinates in OCT and PPObj exhibit poor generalization, with performance dropping by 30\% and 62\%, respectively. \relcoords{} helps mitigate this, with PPObj using player-centric coordinates, reducing the drop to 7\% and ORBiT dropping by only 4\%, both comparable to \ppocnn{} (7\%).

These results highlight the biggest advantages of using \relcoords{} with a transformer. 
\relcoords{} prevents catastrophic failure when the information in the game screen is slightly shifted in the observations (\variant{translate}).

\subsection{Continuous Control and Anchor Choice}
\begin{figure}
    \centering
    \begin{subfigure}[b]{0.46\linewidth}
        \centering
        \includegraphics[width=\linewidth]{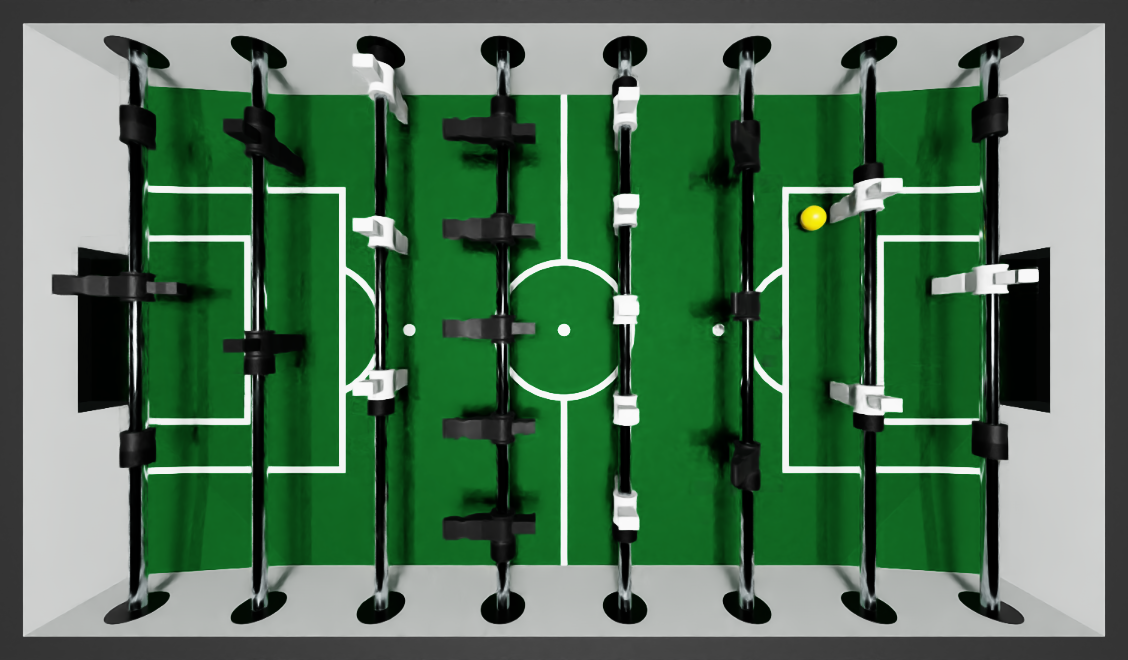}
        \vspace{16mm}
        \caption{The \framework{Foosball} environment}
    \end{subfigure}%
    \hfill
    \begin{subfigure}[b]{0.50\linewidth}
        \centering
        \includegraphics[width=\linewidth]{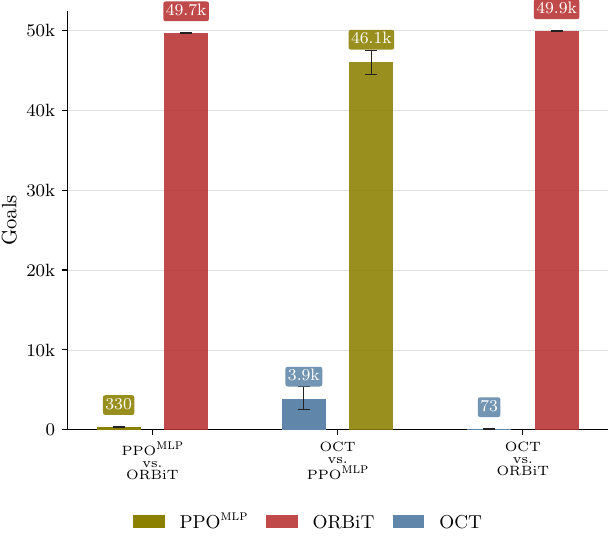}
        \caption{Head-to-head goals}
    \end{subfigure}%
    \caption{\textbf{ORBiT consistently dominates in the \framework{Foosball} environment (a).} (b) Both \ppomlp{} and OCT are decisively outperformed by ORBiT; they score only 330 and 73 goals out of 50k. With \relcoords{}, ORBiT learns effective policies, while OCT (absolute coordinates) remains far behind \ppomlp{}. Full per-match numbers are in \cref{app:ext}.}
    \label{fig:foosball}
\end{figure}

~
\vspace{-0.8cm}
\paragraph{RQ4 --- Beyond Atari,  relative coordinates (\relcoords{}) also boost continuous-control tasks.}

To investigate this, we have used ORBiT on an additional task: \game{Foosball}~(see \citet{foosball}, \cref{fig:foosball}).

We trained the agents to play the full game of \game{Foosball} (all rods with all figurines) with 4 rods per player using the \emph{mixed self-play setup} (see \cref{app:foosball}) and the ball as the anchor object. All algorithms, i.e., ORBiT, OCT, and \ppomlp{}, were trained on 5 different random seeds, producing 15 \game{Foosball} agents. The details of the training setup are in \cref{app:foosball}.

To evaluate the different approaches, we compare them by pairing every seed from one algorithm with every seed from the other. Each match-up is 1,000 goals per side (black/white), totaling 50,000 goals per algorithm pair. The results are shown in \cref{fig:foosball}b. We observe that \ppomlp{} scored only 330 goals against ORBiT, while ORBiT shot approximately 49.7k goals, accounting for over 99\% of all possible goals. The results of ORBiT (\relcoords{}) vs. OCT (absolute coordinates) are even clearer, with OCT only able to score 73 times out of 50,000.
This not only shows the strong benefit of \relcoords{} over absolute coordinates, but also ORBiT's dominance in the Foosball task. Neither baseline approach comes even close to ORBiT's strong policy. Detailed results of the different match-ups are in \cref{app:ext}.

To the best of our knowledge, ORBiT enables the first autonomous agent trained from scratch via self-play to play the full game of \game{Foosball}. This demonstrates that ORBiT can scale effectively to continuous, highly dynamic physical domains with complex multi-body coordination. These results suggest that \relcoords{} and ORBiT are promising candidates for real-world robotic control, though validating them on physical hardware with imperfect object detection remains future work.

\paragraph{RQ5 --- Choosing a meaningful anchor is necessary for ORBiT to work.}

Atari and \textit{Foosball} fix the orientation of the objects, the number of objects, and the anchor; \game{Tag} lets us vary each in isolation.

\begin{figure}
    \centering
        \begin{subfigure}[b]{0.56\linewidth}
        \centering
        \includegraphics[width=0.4\linewidth]{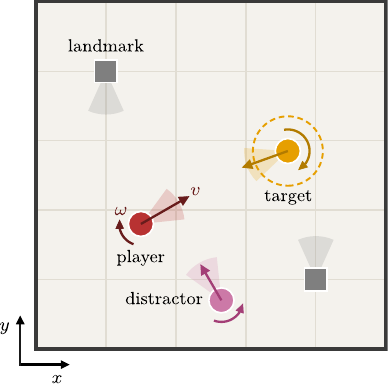}
        \hspace{3mm}
{        \includegraphics[width=0.5\linewidth]{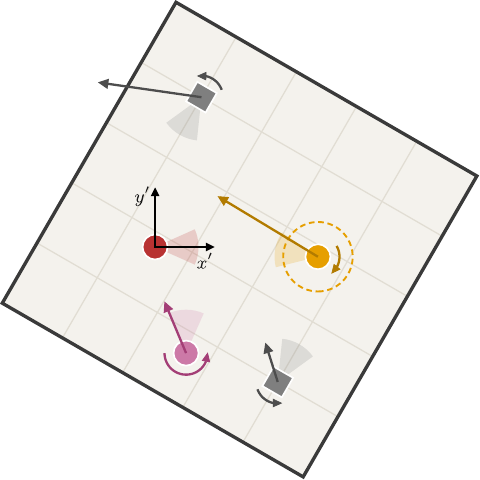}
        \vspace{11mm}}
        \caption{From Absolute to \relcoords{}}
    \end{subfigure}%
    \hspace{2mm}
    \begin{subfigure}[b]{0.4\linewidth}
        \centering
        \includegraphics[width=\linewidth]{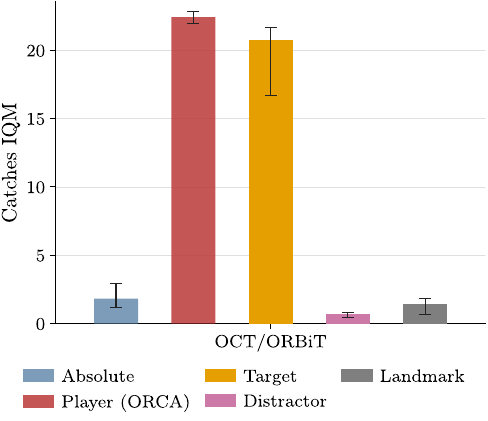}
        \caption{Results for Different Anchor Choices}
    \end{subfigure}%
    \caption{\textbf{In \game{Tag}, ORBiT outperforms absolute coordinates, but the anchor choice matters.} The \relcoords{} observations compute all object observations relative to an anchor object. (a) For example, selecting the player as the anchor transforms absolute (left) into relative (right) representations via a kinematic transformation that sets the player as a static origin (aligned with its heading) and applies the player's linear velocity $v$ and angular velocity $\omega$ in reverse to all objects. The headings, shown as shaded cones, match the direction of motion in the absolute frame but not in the player's frame: every object inherits the player's reversed velocity, plus a sideways velocity due to its rotation. Even static landmarks thus appear to move, and the target's velocity points past the player, although the target heads almost straight at it. The figure is a schematic, with arrow lengths and turn arcs not drawn to scale. (b) The performance results vary strongly for different anchor choices. Choosing an important game object, such as the player, clearly dominates a randomly moving anchor (distractor).}
    \label{fig:tag}
\end{figure}

ORBiT accepts sets of any size by construction, but accepting a set of unseen size does not mean it performs well on it. Since distractors differ from targets only in type (represented by color), an agent that attends to targets should not care about how many distractors surround it. We train with $m = 5$ distractors and evaluate with $m' \in \{0, 1, 5, 10\}$, including more distractors than seen during training. PPObj concatenates a fixed number of object slots ($N_{\max} = 9$: 1 player, 1 target, 2 landmarks, 5 distractors) and cannot be evaluated beyond it. ORBiT's number of catches stays the same for every $m'$, whereas OCT slightly falls with the number of distractors (see \cref{app:ext}).

\Cref{thm:invariance} holds for any object that is part of the scene. We therefore train ORBiT with four different anchors: the player (default), the target, a distractor, and a static landmark at $\mathbf{p}^L$, and measure the influence of anchor choice. The landmark anchor is the informative case. Since $\mathbf{p}^i_t - \mathbf{p}^L$ differs from absolute coordinates only by a constant, it preserves invariance to frame shifts. The player and target anchors reach IQMs of 22.4 and 20.7, the distractor anchor 0.7, and the landmark anchor 1.4, close to OCT (\cref{fig:tag}b). This highlights the importance of the choice of anchor. 

Together, ORBiT is unaffected by the number of distractors, while the choice of anchor decides the amount of improvement through \relcoords{}. PPObj results are in \cref{app:ext}.

\section{Discussion}
\label{sec:discussion}
An agent receiving world-frame positions must learn that the same relative configuration, i.e., a cauldron 5 pixels to the right of the player in \game{Asterix}, is equivalent everywhere it can occur, and thus the agent must learn it again at every location. This is evident in the results of the \variant{translated}-frame experiment. 
\relcoords{} states the equivalence directly: two scenes differing by a global translation produce identical inputs.
The robustness results follow from the same property, since perturbations that shift absolute coordinates without changing relative structure are invisible to an \relcoords{} policy and catastrophic for one trained on world-frame inputs.

We showed this in three domains. In \game{Atari}, PPObj and OCT benefit from \relcoords{}, with ORBiT being the strongest agent. In \game{Foosball} and \game{Tag}, this is even more prominent. We expect the advantage of \relcoords{} to scale with how often a task reuses relative configurations across the world frame, and to shrink wherever the world frame is itself the task structure.

We do not claim this inductive bias is novel: anchoring observations to the acting agent has proven effective in multi-agent RL~\citep{shang2021agent}. Object-centric RL, however, has largely retained absolute coordinates, feeding object features to the policy in a fixed global frame. Our results show that revisiting this input convention yields substantial gains in both performance and robustness.

\paragraph{Broader Applicability.}
\relcoords{} is a drop-in observation wrapper, with no change to the policy architecture, the training algorithm, or the reward. As such, any pipeline that consumes absolute-coordinate object vectors can adopt it by substituting a single transformation at the input. Furthermore, our \game{Foosball} results show that ORBiT handles fast, continuous control in simulation. Although a game, \game{Foosball} shares key demands with real-world robotics, including fast system dynamics and 8-DoF continuous action execution. This makes ORBiT a promising candidate for robotic and other continuous-control domains, pending validation on physical hardware with imperfect object detection.

\paragraph{Limitations and Future Work.}
First, our current implementation is constrained to 2D environments; extending this approach to 3D continuous spaces will be necessary for broader real-world robotic deployment. Second, our state formulation evaluates coordinates relative to a single target entity. This anchor is currently selected manually per environment based on domain knowledge; learning to identify the most decision-relevant anchor automatically would remove this design choice and make the approach more readily transferable to new tasks. Our evaluation in \game{Tag} highlights this matter; with an anchor that is not decision-relevant, most of the gain disappears. Looping back to relational RL, i.e., modeling inter-object spatial relationships across all entities in the scene, is a promising direction for enhancing relational reasoning in the context of relative coordinate frames. Further, we assume a perfect object extractor in our work. The influence of detection quality should be studied in general for OCRL.
Finally, in most of our experiments, we restricted our inputs to positional coordinates and velocities only. Studying richer feature spaces, such as object bounding dimensions, dynamic properties, or semantic labels, could unlock further improvements in OCRL.

\section{Related Work}
\label{sec:related_work}

\paragraph{Object-Centric Representations in deep RL.}
Unlike traditional pixel-based inputs or global feature vectors, object-centric views abstract the environment into discrete entities, enabling relational reasoning and compositional learning. Foundational methods such as Slot Attention~\citep{locatello2020object} have enabled unsupervised object discovery, while platforms like OCAtari~\citep{delfosse2024ocatari} provide structured object-level access to classic benchmarks. Empirical studies report the benefits of such representations for policy learning~\citep{ye2025learning,ye2026boosting} and transferability~\citep{yoon_investigation_2023, DidolkarZGMBMS25}. More recently, \citet{ugadiarov2025relational} combine actor-critic learning with an object-centric world model to exploit relational structure in RL. In contrast, ORBiT is model-free and focuses on the representational choice itself: objects are encoded relative to a task-relevant anchor rather than in absolute coordinates. While most approaches encode objects in a global coordinate frame, they frequently overlook the perspective-dependent structure that often guides agent behavior.

\paragraph{Relational Reinforcement Learning and  Object-Oriented MDPs.} Similar to ORBiT, Relational Reinforcement Learning~\citep{dvzeroski2001relational,van2004challenges,kersting2004bellman} and Object-Oriented MDPs (OO-MDPs)~\citep{diuk2008object}, address the following problem: how to exploit objects and their relations so an RL agent can generalize across environments and different numbers of objects. The main difference from ORBiT is their focus: classical Relational RL, on the one hand, typically uses first-order logic or relational rules to describe relations. OO-MDP, on the other hand, centers around the object interactions within the transition model.
In contrast, ORBiT instead utilizes an egocentric view, encoding objects relative to a specified anchor object. 

\paragraph{Transformers for Relational and Object-Centric RL.}
Transformers have become increasingly popular in RL due to their capacity to model relational structure via self-attention. For example, the Set Transformer~\citep{lee2019set} introduced an attention-based architecture for permutation-invariant processing of unordered sets. Early applications, such as the Decision Transformer \citep{chen2021decisiontransformer}, framed RL as a sequence-modeling problem, while others \citep{parisotto20, hu2021updetuniversalmultiagentreinforcement, yang_STT, esslinger_deep_2022} adapted transformers for partially observable or multi-agent settings. 
In object-centric settings, transformers have been used to model interactions between entities \citep{zadaianchuk2020selfsupervisedvisualreinforcementlearning, mambelli2022compositionalmultiobjectreinforcementlearning}, though these models often rely on goal conditioning or absolute positions, thereby limiting their ability to generalize across spatial variations. Additionally, they typically employ aggregation mechanisms that are sensitive to input size or ordering, reducing flexibility in dynamic environments. More recently, \citet{haramatientity} introduced an entity interaction transformer for OCRL that combines object-centric representations with attention-based state, goal, and entity-interaction modeling, enabling multi-object manipulation from pixels and generalizing to more objects than were seen during training. %
ORBiT complements these works. It uses a transformer to reason over sets of objects in relative frames of reference and therefore leans on anchor-relative object representations as an inductive bias to guide and elevate learning.

\paragraph{Inductive Biases for Spatial Generalization.}
A growing body of work explores how inductive biases can enhance generalization in spatially structured tasks. Works on relational architectures, including Graph Networks~\citep{battaglia2018relational}, argue that representing entities, relations, and compositional structure provides a useful bias for systematic generalization. Building on this perspective, transformer-based architectures with explicit relational structure~\citep{altabaa2023abstractors, altabaa2024disentangling} and object-centric attention mechanisms~\citep{bluml2026ocp} have shown promising results in learning efficiency and transfer. These methods highlight the importance of modeling relations between entities, but they typically assume absolute positions in global coordinate frames. Drawing inspiration from cognitive science and neuroscience, recent work has suggested that egocentric inductive biases can improve robustness and behavioral relevance in reinforcement learning systems (see next paragraph). Our work follows this motivation, but instead of relying on explicit graph structure, we use self-attention over object tokens and introduce anchor-relative spatial coordinates as an inductive bias for object-centric RL.

\paragraph{Egocentric and Relative Representations.}
The distinction between egocentric (agent-relative) and allocentric (world-relative) representations has been extensively studied in cognitive science and neuroscience, with evidence that biological agents frequently operate in egocentric frames to support robust spatial reasoning \citep{klatzky1998allocentric, galati2000neural, burgess2006spatial}. Similar principles have informed navigation and control in robotics, where encoding observations relative to an agent or task-relevant anchor improves sample efficiency and generalization \citep{hu2025egocentric}.
Despite these precedents, most RL approaches continue to rely on absolute spatial encodings.
ORBiT also operates on sets of object observations, but our focus is on RL with anchor-relative object representations. Agent-centric methods have been explored in multi-agent RL, where representing observations relative to the acting agent improves sample efficiency and generalization~\citep{shang2021agent}. Our work brings this idea to single-agent object-centric RL by encoding each object relative to a task-relevant anchor and combining this with a transformer-based policy.

\paragraph{Playing Foosball.} 
Aside from the methodological part of this work, we test our approach on the game of \game{Foosball}.
\Citet{foosball} identify Foosball as a challenging robotic environment and an important stepping stone towards more complex robotics applications. In recent years,  a couple of works have automated a Foosball table and created a digital twin~\citep{soopen,croenen276604268reinforcement,rohrer2021foosball,de2021kicker}, highlighting the interest in, and challenges of this task.
While the idea of an automated Foosball table is not new~\citep{weigel2002kiro}, the majority of Foosball agents are still rule-based and hand-coded~\citep{ZaloznikSoln2024}. As training the full game with 4 rods from scratch is demanding, much of the work focuses on learning individual skills such as scoring, passing, and blocking~\citep{foosball,croenen276604268reinforcement}. In contrast, in this work, we use ORBiT to learn to play \game{Foosball} from scratch via self-play using the PPO algorithm. For this, we base our implementation on the work of \citet{foosball} and introduce mixed self-play, a form of domain randomization~\citep{domain_rand}, to boost early-stage training by enabling training across highly varied strategic scenarios.

\section{Conclusion}
Object-centric RL uses the fixed world frame as its dominant coordinate system. We revisit this convention and present an alternative.
\Relcoordslong{} expresses each object's position and features in the frame relative to a task-relevant anchor object. \relcoords{} does not introduce any learnable parameters, nor impose any constraints on the approaches for which this input representation can be used, except that access to the objects themselves is required. We pair \relcoords{} with a transformer to create ORBiT.

Across our test set of \numgames{} \framework{Atari} games, where we replaced absolute coordinates with \relcoords{}, both the transformer and the MLP outperformed their baselines (see \cref{fig:aggregate}). Under zero-shot perturbations, described in \cref{sec:experiments}, we demonstrate that \relcoords{} and ORBiT both improve robustness in out-of-distribution examples.

Beyond \framework{Atari}, ORBiT learns the full game of \game{Foosball}, an $8$-DoF continuous-control task, from scratch via self-play. Against it, PPO$^{\text{MLP}}$ scores only $330$ and OCT only $73$ of $50{,}000$ goals. In \game{Tag}, where the anchor carries an orientation, ORBiT reaches $22.4$ catches per episode versus $1.8$ for OCT. Although trained with $5$ distractors, it maintains this performance across different numbers of distractors from $0$ to $10$. Furthermore, it shows the importance of the choice of anchor, with semantically meaningful objects like the player object or the target showing better performance than static landmarks or distractors.

Together, these results suggest that object-centric RL should use relative rather than absolute coordinates whenever the task provides a meaningful anchor. Selecting the anchor automatically, combining several anchors, and extending \relcoords{} to 3D environments and imperfect object detectors are the next steps towards extending ORBiT.

\subsubsection*{Acknowledgments}
This research work has been funded by the Deutsche Forschungsgemeinschaft (DFG, German Research Foundation) under Germany's Excellence Strategy (EXC-3057/1 ''Reasonable Artificial Intelligence``, Project No. 533677015).

\bibliography{clean}
\bibliographystyle{tmlr}

\appendix
\crefalias{section}{appsec}
\clearpage

\section{Evaluation Metrics}
\label{app:metrics}
To evaluate the agents' performance, we use three metrics following best practices from \citet{agarwal2021deep}: (1) \textbf{Human-Normalized Score (HNS)}, (2) \textbf{Interquartile Mean (IQM)}, and (3) \textbf{95\% Confidence Intervals (CI)}.

\paragraph{Human-Normalized Score (HNS).} 
To compare performance across different \framework{ALE} environments with varying reward ranges, we use HNSs introduced by \citet{mnih2015human}, defined via mean agent score $S_A$, the human score $S_H$, and the random score $S_R$:
\begin{equation}
\text{HNS} = \frac{S_A - S_R}{|S_H - S_R|}.
\end{equation}
A value of 1.0 resembles human-level performance, while values near 0 indicate behavior close to random. Scores exceeding 1.0 imply superhuman performance. The human and random scores are taken from \citet{BadiaPKSVGB20}.

\paragraph{Aggregation via Interquartile Mean (IQM).} 
To ensure robustness against outliers, we aggregate scores using the \textbf{IQM}, which computes the mean over the middle 50\% of the data (the 25th to 75th percentiles). For our experiments, we report the IQM computed using the \variant{rliable} library~\citep{agarwal2021deep}.

\section{Proof of Rigid Invariance}
\label{app:invariance_proof}

We give a formal proof of the claim made in \cref{sec:aor}: When each object position $\mathbf{p}^i$ is expressed relative to an anchor position $\mathbf{p}^\psi$, the agent's policy is invariant to rotations $\mathbf{R}_g$ and translations $\mathbf{b}_g$ of the coordinate system.
The theorem below makes this precise.

\subsection{Setup}

Let $d$ denote the spatial dimension of the environment; $d = 2$ in all of our experiments.
We write $\mathit{SO}(d) = \{\mathbf{Q} \in \mathbb{R}^{d \times d} : \mathbf{Q}^\top \mathbf{Q} = \mathbf{I},\ \det \mathbf{Q} = 1\}$ for the group of rotations about the origin, and
\begin{equation}
  \mathit{SE}(d) \;=\; \mathit{SO}(d) \ltimes \mathbb{R}^d
\end{equation}
for the special Euclidean group of rigid motions, whose elements $g = (\mathbf{R}_g, \mathbf{b}_g)$ pair a rotation $\mathbf{R}_g \in \mathit{SO}(d)$ with a translation $\mathbf{b}_g \in \mathbb{R}^d$. 
Concretely, for $d = 2$,
\begin{equation}
  \mathit{SO}(2) = \left\{ \mathbf{R}(\alpha) = \begin{pmatrix} \cos\alpha & -\sin\alpha \\ \sin\alpha & \cos\alpha \end{pmatrix} : \alpha \in [0, 2\pi) \right\},
\end{equation}
the rotations of the plane about the origin, and $\mathit{SE}(2)$ additionally allows the origin to move.
Every element of $\mathit{SE}(2)$ is only a rotation followed by a shift; hence,  rigid motions preserve all distances and angles between points, which is why they leave the relations among objects intact while changing every absolute coordinate.

At time $t$, the agent receives $\mathcal{O}_t = \{o^i_t\}_{i=1}^{N_t}$, where $o^i_t = (\mathbf{c}^i_t, \mathbf{p}^i_t, \mathbf{v}^i_t)$ are intrinsic attributes, position, and velocity in the world frame.
The anchor $\psi_t$ has position $\mathbf{p}^\psi_t$, velocity $\mathbf{v}^\psi_t$, and orientation $\boldsymbol{Q}_t \in \mathit{SO}(d)$.
Given $\mathbf{R}_t = \boldsymbol{Q}_t^\top$, the \relcoords{} transformation is
\begin{align*}
T(\psi_t, o^i_t)
&=
\bigl[\, \mathbf{c}^i_t,\;
\mathbf{R}_t \cdot (\mathbf{p}^i_t - \mathbf{p}^\psi_t),\;
\mathbf{R}_t \cdot (\mathbf{v}^i_t - \mathbf{v}^\psi_t) \,\bigr] \\
\end{align*}
and the \relcoords{} observation is $\tilde{\mathcal{O}}_t = \{T(\psi_t, o^i_t)\}_{i=1}^{N_t}$.

Let a rigid transformation $g = (\mathbf{R}_g, \mathbf{b}_g) \in \mathit{SE}(d) = \mathit{SO}(d) \ltimes \mathbb{R}^d$ act by
\begin{align*}
g \cdot \mathbf{c} &= \mathbf{c}
  &&\text{(intrinsics are coordinate-free)},\\
g \cdot \mathbf{p} &= \mathbf{R}_g\mathbf{p} + \mathbf{b}_g,\\
g \cdot \mathbf{v} &= \mathbf{R}_g\mathbf{v}.
\end{align*}
The velocity is not shifted because it corresponds to the first derivative of the position, canceling all static translations.
It also rotates the anchor's frame: $g \cdot \boldsymbol{Q}_t = \mathbf{R}_g\boldsymbol{Q}_t$, giving $g \cdot \mathbf{R}_t = \mathbf{R}_t\mathbf{R}_g^\top$.
We write $(g \cdot \psi_t,\, g \cdot \mathcal{O}_t)$ for the scene with $g$ applied jointly to every object and to the anchor.

\subsection{Invariance theorem}

\begin{theorem}[Joint rigid invariance of \relcoords{}]
\label{thm:invariance}
For every $g \in SE(d)$, anchor $\psi_t$, and object set $\mathcal{O}_t$,
\begin{equation}
    \tilde{\mathcal{O}}^g_t = \{T(g \cdot \psi_t, g \cdot o^i_t)\}_{i=1}^{N_t}
    \stackrel{!}{=}
    \{T(\psi_t, o^i_t)\}_{i=1}^{N_t} = \tilde{\mathcal{O}}_t
\end{equation}
\end{theorem}

\begin{proof}
To show this equivalence, it suffices to show that $T(g \cdot \psi_t, g \cdot o^i_t) = T(\psi_t, o^i_t)$ for each $i$. Let us do this by looking at all elements of $o^i_t$ separately:

\emph{Intrinsics.}
$g \cdot \mathbf{c}^i_t = \mathbf{c}^i_t$ by definition. As such, all intrinsic information is invariant.

\emph{Positions.}
Since we apply $\mathbf{b}_g$ to all objects, we can show that $\mathbf {b}_g$ cancels out. As such, we can show that positions are invariant under positional and rotational change:
\begin{align*}
  &(g \cdot \mathbf{R}_t)\bigl((g \cdot \mathbf{p}^i_t) - (g \cdot \mathbf{p}^\psi_t)\bigr)\\
  &= \mathbf{R}_t\mathbf{R}_g^\top
     \bigl((\mathbf{R}_g\mathbf{p}^i_t + \mathbf{b}_g)
          - (\mathbf{R}_g\mathbf{p}^\psi_t + \mathbf{b}_g)\bigr)
     &&\text{(substitute group action)}\\
  &= \mathbf{R}_t\mathbf{R}_g^\top
     \bigl(\mathbf{R}_g\mathbf{p}^i_t - \mathbf{R}_g\mathbf{p}^\psi_t\bigr)
     &&\text{($\mathbf{b}_g$ cancels)}\\
  &= \mathbf{R}_t\mathbf{R}_g^\top\mathbf{R}_g
     \,(\mathbf{p}^i_t - \mathbf{p}^\psi_t)
     &&\text{(factor $\mathbf{R}_g$)}\\
  &= \mathbf{R}_t\cdot(\mathbf{p}^i_t - \mathbf{p}^\psi_t).
     &&\text{($\mathbf{R}_g^\top\mathbf{R}_g = \mathbf{I}$)}
\end{align*}

\emph{Velocities.}
The same cancellation holds without a translation term $\mathbf{b}$:
\begin{align*}
  &(g \cdot \mathbf{R}_t)(g \cdot \mathbf{v}^i_t - g \cdot \mathbf{v}^\psi_t)\\
  &= \mathbf{R}_t\mathbf{R}_g^\top
     (\mathbf{R}_g\mathbf{v}^i_t - \mathbf{R}_g\mathbf{v}^\psi_t)
     &&\text{(substitute group action)}\\
  &= \mathbf{R}_t\mathbf{R}_g^\top\mathbf{R}_g
     \,(\mathbf{v}^i_t - \mathbf{v}^\psi_t)
     &&\text{(factor $\mathbf{R}_g$)}\\
  &= \mathbf{R}_t\cdot(\mathbf{v}^i_t - \mathbf{v}^\psi_t).
     &&\text{($\mathbf{R}_g^\top\mathbf{R}_g = \mathbf{I}$)} \qedhere
  \end{align*}

From this $T(g \cdot \psi_t, g \cdot o^i_t) = T(\psi_t, o^i_t)$ follows.
\end{proof}

\subsection{Policy invariance}

\begin{corollary}
\label{cor:policy-equivariance}
Any policy $\pi(a \mid \tilde{\mathcal{O}}^g_t)$ satisfies
\[
\pi\!\left(\cdot \mid \tilde{\mathcal{O}}^g_t\right)
\;=\;
\pi\!\left(\cdot \mid \tilde{\mathcal{O}}_t\right)
\]
for all $g \in SE(d)$; the value function $V$ inherits the same property.
\end{corollary}
\begin{proof}
The two inputs are identical by \cref{thm:invariance}.
\end{proof}

Two scenes related by any rigid transform produce the same action distribution, with no data augmentation or auxiliary loss required.

\subsection{Permutation invariance}

Rigid invariance and permutation invariance are independent properties handled at different layers.
\relcoords{} is applied element-wise, so $\tilde{\mathcal{O}}_t$ is an unordered set; whether the policy respects that structure depends on the encoder.
ORBiT uses a transformer with mean-pooling, which is permutation-invariant by construction. MLPs are not.
Losing either guarantee can cause catastrophic failure. Agents without \relcoords{} fail under \variant{translations}; agents without a set-equivariant encoder fail under \variant{shuffles} (\cref{fig:robustness}).
ORBiT satisfies both.

\begin{remark}[What the experiments exercise]
\label{rem:translation-only}
The Atari player objects and the ball in Foosball have no meaningful orientation, so we set $\mathbf{Q}_t = \mathbf{R}_t = \mathbf{I}$ throughout those experiments.
\Cref{thm:invariance} then only looks at $\{(\mathbf{I}, \mathbf{b}): \mathbf{b} \in \mathbb{R}^d\} \cong \mathbb{R}^d$, and \relcoords{} reduces to subtracting the anchor's position and velocity from every object.
Every \game{Atari} and \game{Foosball} result is, therefore, evidence for translation invariance only; the rotational half of the theorem is exercised only where the anchor carries an orientation, e.g., the \game{Tag} experiments.
\end{remark}

\section{The ORBiT Architecture}
\label{app:transformer}
This section gives details about the architecture used for ORBiT.

\paragraph{Transformer encoder.}
The encoder has $\ell$ stacked layers, each with multi-head self-attention (MHSA), a position-wise feedforward network (FFN), residual connections, and layer normalization~\citep{vaswani2017attention}.
The transformer is a natural fit here: it operates on sets rather than sequences, so no positional encodings are needed, and the encoder is permutation-equivariant by construction.
It also handles any $N_t$ without modification.

\paragraph{Aggregation}
The transformer encoder layers output as many tokens as objects in the input. Therefore, we need to aggregate them into a single vector before we can input them into the next neural network layers. There are several options to do this: \textit{max}-aggregation, which takes the elementwise maximum of the tokens, \textit{mean}-aggregation, which takes the elementwise mean of the tokens, just using the \textit{first} token (see e.g., \citet{yoon_investigation_2023}), \textit{sum}-aggregation, which takes the elementwise sum of the tokens. While the last one does not scale well with the number of objects, the former ones do.
In our early experiments, we found that \textit{mean}-aggregation worked best, so we used it for all our experiments.

\paragraph{Policy and Value Heads}
The aggregated vector feeds two linear heads: a policy head that produces a distribution over actions and a value head that produces a scalar state-value estimate.
The full network trains end-to-end with PPO~\citep{schulman2017PPO}, without auxiliary losses.
Hyperparameter details ($\ell$, $d$, attention heads, PPO schedule) are in \cref{app:hyperparams}.

\section{Hyperparameters and Code}
\label{app:hyperparams}

Key hyperparameters for our experimental setups are given in \cref{tab:hardware,tab:hyperparams_atari,tab:hyperparams_foosball,tab:hyperparams_orbits,tab:hyperparams_ppo}. The source code to replicate our findings is available at \url{https://anonymous.4open.science/r/ORBiT_TMLR} and will be made publicly available after acceptance.

\begingroup
\rowcolors{2}{}{shade}
\begin{table}[b]
    \centering
    \caption{Key hyperparameters regarding the used \textbf{Atari environments}. We are using Gymnasium and \framework{ALE}, following best practices from \citet{MachadoBTVHB18} and the community.}
    \setlength{\tabcolsep}{3pt}  %
    \renewcommand{\arraystretch}{1.1}  %
    \begin{tabular}{|l|c|}
        \hline
        \textbf{Hyperparameter}       & \textbf{Value} \\
        \hline
        ALE version                   & 0.8.1          \\
        Gymnasium version             & 0.29.1         \\
        Environment version           & v5             \\
        Frameskip                     & 4              \\
        Buffer Window Size            & 4              \\
        Observation Mode              & RGB            \\
        Repeat Action Probability     & 0.25           \\
        Full Action Space             & \variant{False} \\
        Continuous                    & \variant{False} \\
        \hline
        Number of Evaluation Episodes & 10 \\
        \hline
    \end{tabular}
    \label{tab:hyperparams_atari}
\end{table}
\endgroup

\begingroup
\rowcolors{2}{}{shade}
\begin{table}[ht]
    \centering
    \caption{Key hyperparameters used for \textbf{ORBiT}.}
    \label{tab:hyperparams_orbits}
    \setlength{\tabcolsep}{3pt}  %
    \renewcommand{\arraystretch}{1.1}  %
    \begin{tabular}{|l|c|}
    \hline
    \textbf{Hyperparameter} & \textbf{Value} \\
        \hline
        Number of Transformer Blocks ($\ell$) & $3$ \\
        Embedding Dimension ($d$) & $128$ \\
        Number of Attention Heads & 8 \\
        Include Width and Height & \texttt{False} \\
        \hline
    \end{tabular}
\end{table}
\endgroup

\begingroup
\rowcolors{2}{}{shade}
\begin{table}[t]
    \centering
    \caption{Key hyperparameters regarding the used \textbf{Foosball environments}.}
    \setlength{\tabcolsep}{3pt}  %
    \renewcommand{\arraystretch}{1.1}  %
    \begin{tabular}{|l|c|}
        \hline
        \textbf{Hyperparameter}   & \textbf{Value}       \\
        \hline
        Gym version               & 0.23.1               \\
        Task                      & FoosballMixedSelfplay \\
        DistToBallReward          & \variant{True}        \\
        WinReward                 & 150                  \\
        LossPenalty               & 150                  \\
        StagnationPenalty         & 150                  \\
        TerminationPenalty        & 150                  \\
        TimeoutPenalty            & 0                    \\
        \hline
    \end{tabular}
    \label{tab:hyperparams_foosball}
\end{table}
\endgroup

\begingroup
\rowcolors{2}{}{shade}
\begin{table}[hb!]
    \centering
    \caption{Key hyperparameters used for \textbf{training}, following \citet{shengyi2022the37implementation,huang2022cleanrl,foosball}. We adopt these parameters without additional tuning to ensure a fair comparison across all methods.}
    \label{tab:hyperparams_ppo}
    \setlength{\tabcolsep}{3pt}  %
    \renewcommand{\arraystretch}{1.1}  %
    \begin{tabular}{|l|c|c|}
    \hline %
    \textbf{Hyperparameter}                  & \textbf{Value (Atari)} & \textbf{Value (Foosball)} \\
    \hline %
    Seed                                     & \multicolumn{2}{|c|}{\{0,\,1,\,2,\,3,\,4\}}        \\
    Learning Rate ($\alpha$)                 & $2.5 \cdot 10^{-4}$    & $3 \cdot 10^{-4}$         \\
    Total Timesteps$^\dagger$/Max Epochs$^\ddagger$     & $^\dagger10^7$         & $^\ddagger6 \cdot 10^7$            \\
    Number of Environments                   & $10$                   & $4{,}096$                 \\
    Batch Size ($B$)                         & $1{,}280$              & $4{,}096$                 \\
    Minibatch Size ($b$)                     & $320$                  & $4{,}096$                 \\
    Update Epochs                            & $4$                    & $1$                       \\
    GAE Lambda ($\lambda$)                   & $0.95$                 & $0.95$               \\
    Discount Factor ($\gamma$)               & $0.99$                 & $0.99$                    \\
    Value Function Coefficient ($c_v$)       & $0.5$                  & $4.0$                     \\
    Entropy Coefficient ($c_e$)              & $0.01$                 & $0.0001$                  \\
    Clipping Coefficient ($\epsilon$)        & $0.1$                  & $0.2$                     \\
    Clip Value Loss                          & \variant{True}          & \variant{True}             \\
    Max Gradient Norm ($\|g\|_{\text{max}}$) & $0.5$                  & $1.0$                     \\
    \hline %
    \end{tabular}
\end{table}
\endgroup

\begingroup
\rowcolors{2}{}{shade}
\begin{table}[htb!]
    \centering
     \caption{\textbf{Hardware configuration} used for our experiments. We used NVIDIA DGX systems and followed standardized training pipelines.}
    \setlength{\tabcolsep}{3pt}
    \renewcommand{\arraystretch}{1.1}
    \begin{tabular}{|l|c|c|}
        \hline
        \textbf{Component} & \textbf{Specification (Atari)}   & \textbf{Specification (Foosball)} \\
        \hline
        System             & NVIDIA DGX v.5.1.0               & NVIDIA DGX v.7.5.0                \\
        GPUs               & 16 $\!\times\!$ NVIDIA V100 32GB & 8 $\!\times\!$ NVIDIA A100 40GB   \\
        CPUs               & Intel Xeon Platinum 8174         & AMD EPYC 7742                     \\
        System Memory      & 1.58 TB                          & 2.05 TB                           \\
        Operating System   & Ubuntu 20.04 LTS                 & Ubuntu 24.04   LTS                \\
        CUDA Version       & 12.4                             & 13.0                              \\
        Python Version     & 3.10                             & 3.10                              \\
        Docker Base Image  & nvcr.io/nvidia/pytorch:23.09-py3 & nvcr.io/nvidia/isaac-sim:4.0.0    \\
        \hline
    \end{tabular}
    \label{tab:hardware}
\end{table}
\endgroup

\clearpage
\section{Atari Games}
\label{app:games}

We trained all approaches in \numgames{} Atari 2600 games from the Arcade Learning Environment (ALE), training each algorithm with 5 random seeds: \game{Alien}, \game{Amidar}, \game{Asterix}, \game{Boxing}, \game{Breakout}, \game{Freeway}, \game{Frostbite}, \game{Gopher}, \game{Kangaroo}, \game{MsPacman}, \game{Pong}, \game{SpaceInvaders}, and \game{Tennis}.

\paragraph{HackAtari Modifications.} After training in the original environment, we also evaluated each agent in modified versions of these to test generalization using the \framework{HackAtari} framework~\citep{hackatari}. \Cref{tab:modifs} lists the used variants.

\begin{table}[t]
  \centering
  \caption{The \textbf{HackAtari modifications} used to evaluate generalization.}
  \label{tab:modifs}
  \begin{tabular}{ll}
    \toprule
    \textbf{Game} & \textbf{HackAtari Modification} \\
    \midrule
    \game{Alien} & \variant{last\_egg} \\
    \game{Amidar} & \variant{pig\_enemies} \\
    \game{Boxing} & \variant{switch\_position} \\
    \game{Breakout} & \variant{gravity} \\
    \game{Freeway} & \variant{stop\_all\_cars\_tunnel} \\
    \game{Frostbite} & \variant{reposition\_floes\_easy} \\
    \game{Kangaroo} & \variant{randomize\_kangaroo\_position} \\
    \game{MsPacman} & \variant{set\_level\_2} \\
    \game{Pong} & \variant{hidden\_enemy} \\
    \game{SpaceInvaders} & \variant{disable\_shield\_left} \\
    \bottomrule
  \end{tabular}
\end{table}

\section{Foosball}
\label{app:foosball}
For the \game{Foosball} task (see \cref{fig:foosball}), we build on the implementation of \citet{foosball}. However, we introduce some adaptations to make the environment fit for our use case. Firstly, we introduce \textit{mixed self-play} (FoosballMixedSelfplay). In this training setup, the Foosball table's rod configuration is randomized. The first Foosball table always has the full setup (i.e., the agent and the opponent control all 4 rods), but the other tables can vary. 
While the agent always controls at least one rod, the opponent does not; for example, the agent's keeper rod can be the only active rod on the table.
However, if there are opponent rods, the agent always controls a rod closer to its own goal than the opponent's offense-most rod. The ball is given randomly to any of the figurines on the active rods such that both agent and opponent may be in control of the ball at the beginning of an episode. This setup provides experiences in highly varied strategic scenarios.
Mixed self-play was used only for training; we always evaluated on the full game. 

The observations for \ppomlp{} are the same as defined in the original paper, i.e., the joint positions and velocities, as well as the ball's X- and Y positions.

In the object-centric setting, the agent can observe the ball, the goals, and all figurines of each team independently. 
Each object is passed to the agent as an 11d vector that consists of the following properties:
\begin{itemize}
    \item a one-hot embedding of the object type (5d),
    \item the object's X- and Y positions,
    \item the object's X- and Y velocities,
    \item the object's rotation around the Y axis (only non-zero for the figurines),
    \item the object's rotational velocity around the Y axis (only non-zero for the figurines).
\end{itemize}

\section{Tag}
\label{app:tag}

\game{Tag} is an object-centric 2D environment in a bounded square arena of side length $10$. Each episode contains the player, the target, five distractors, and two static landmark objects. The agent's task is to catch the moving target object. A catch is registered when the player comes within 1 unit of the target. On each catch, the agent receives a reward of $+1$, and the target, as well as the five distractors, are randomly redistributed in the environment more than 3 units away from the player, each with a random heading and a speed drawn uniformly from $[0, 0.5]$. The landmarks are positioned at $(2.5, 2.5)$ and $(7.5,7.5)$ with opposing heading angles. The episode ends either on catch or after 500 environment steps during training; during evaluation, it ends after 500 environment steps, and we report the number of catches made within this time frame.

\paragraph{Observations.} Each object is represented by a vector consisting of a one-hot type (player, target, distractor, landmark), its position, its orientation (heading angle), its speed, and its angular velocity. The target and the distractors differ only in type.

\paragraph{ORCA.} Since objects in \game{Tag} also have a heading and an angular velocity, ORCA uses the full kinematic transform into the co-moving and co-rotating frame of the anchor:
\begin{align*}
    x' &= (x - x_a)\cos\theta_a + (y - y_a)\sin\theta_a, & y' &= -(x - x_a)\sin\theta_a + (y - y_a)\cos\theta_a, \\
    v'_{x} &= v \cos(\theta - \theta_a) - v_a + \omega_a y', & v'_{y} &= v \sin(\theta - \theta_a) - \omega_a x', \\
    \theta' &= \operatorname{wrap}(\theta - \theta_a), & \omega' &= \omega - \omega_a.
\end{align*}
Here, the index $a$ corresponds to the anchor's parameters. Note that we need to split the velocity into two components $(v'_{x}, v'_{y})$ when applying the ORCA transform because the velocity no longer aligns with the heading angle of the object.

\paragraph{Dynamics and actions.} All moving objects are implemented as unicycles controlled by a continuous throttle $a_v$ and a steering action $a_\omega$ with $a_v, a_\omega\in[-1,1]$. Let $v$ be the speed in the direction of the object's heading, $\omega$ the angular velocity, $\theta$ the heading angle, and $x,y$ the position in the absolute coordinate frame. Then, on each environment step, these properties are updated as follows:
\begin{equation*}
    v \gets \operatorname{clip}(v + 0.05a_v, 0, 0.5), \qquad
    \omega \gets \tfrac{\pi}{8}a_\omega, \qquad
    \theta \gets \operatorname{wrap}(\theta + \omega), \qquad
    x \gets x + v \cos\theta, \qquad
    y \gets y + v \sin\theta
\end{equation*}
for each moving object, where $\operatorname{clip}(z, z_\text{min}, z_\text{max})$ limits $z$ to the interval $[z_\text{min}, z_\text{max}]$ and $\operatorname{wrap}$ maps an angle to $[-\pi,\pi)$. While the player's position is clipped at the walls, the target and distractors bounce off them: their positions are reflected back into the arena, and their heading angles are mirrored (and, with this, the implicit velocity vector). All objects preserve their speed and angular velocity in those collisions. Further, the target and distractor actions are selected via
\begin{align*}
    \epsilon_v & \sim \mathcal{N}(0, 0.5^2), & \epsilon_\omega & \sim \mathcal{N}(0, 0.5^2), \\
    u_v & \gets 0.9 u_v + \sqrt{1 - 0.9^2}\, \epsilon_v, & u_\omega & \gets 0.9 u_\omega + \sqrt{1 - 0.9^2}\, \epsilon_\omega, \\
    \varphi & \gets \operatorname{atan2}(y - y_\text{player}, x-x_\text{player}), & \tilde{\varphi} & \gets \operatorname{clip}(\operatorname{wrap}(\varphi - \theta) / \tfrac{\pi}{8}, -1, 1), \\
    a_v & \gets \operatorname{clip}(0.5u_v + 0.5, -1, 1), & a_\omega & \gets \operatorname{clip}(0.5u_\omega + 0.5\tilde{\varphi}, -1, 1)
\end{align*}
at each timestep, which corresponds to a smooth, random control that actively avoids the player. Whenever an object is (re)spawned, $u_v$ and $u_\omega$ are drawn from $\mathcal{N}(0, 0.5^2)$. Within a step, the player moves first, so $x_\text{player}$ and $y_\text{player}$ already refer to its new position. 

\paragraph{Experiments.} In our experiments, we train for $3 \cdot 10^6$ timesteps and use 10 parallel environments. \cref{tab:tag-results} lists the number of catches within 500 environment steps after training for different objects as anchors.

\clearpage
\section{Extended Results}
\label{app:ext}
 This section provides further results on the \game{Foosball} task, i.e., more evaluations of all possible match-ups (see \cref{tab:ext_foosball}), and the per-game results for the \game{Atari} plots (\cref{fig:aggregate,fig:drop}) in \cref{tab:ext_atari}. Except in \game{Kangaroo} and \game{Pong}, ORBiT consistently yields improvements over OCT. The lower average in \game{Kangaroo} is driven by a single seed (\cref{tab:kangaroo}), while \game{Pong} likely presents a greater challenge without absolute position information. However, further data is required to draw more definitive conclusions. \Cref{tab:tag-results} displays the influence of different numbers of objects and anchor choices in \game{Tag}.

\begin{table}[bh]
    \setlength{\tabcolsep}{5pt}  %
    \renewcommand{\arraystretch}{1.1}  %
    \centering
    \caption{\textbf{Extended Foosball results: Goals scored by every agent against every other agent.} ORBiT easily defeats the \ppomlp{} agents on all seed combinations, as well as the OCT without \relcoords{}. The values show the number of goals scored by the first agent (top) against the second agent (left) out of 2,000 total goals; e.g., 7 in (a) means \ppomlp{} scored 7 goals while ORBiT scored 1,993. Each agent plays 1,000 goals and then switches sides. The bottom row shows the sum of goals scored by the first agent ($\Sigma_s$), while the last column shows the number of goals conceded by the second agent ($\Sigma_c$).}
    \label{tab:ext_foosball}
    \begin{subtable}{0.47\textwidth}
        \centering
        \caption{\ppomlp{} vs. ORBiT.}
        \begin{tabular}{cccccccc}
            \toprule
            &          & \multicolumn{5}{c}{\textbf{\ppomlp{}}} &     \\
            & seed     & 0  & 1  & 2  & 3  & 4   & $\Sigma_c$     \\
            \midrule
            \multirow{5}{*}{\rotatebox{90}{\textbf{ORBiT}}}
            & 0        & 12 & 9  & 13 & 6  & 11  & 51           \\
            & 1        & 17 & 7  & 14 & 11 & 12  & 61           \\
            & 2        & 24 & 9  & 10 & 8  & 5   & 56           \\
            & 3        & 35 & 18 & 20 & 26 & 25  & 124          \\
            & 4        & 9  & 4  & 14 & 5  & 6   & 38           \\
            & $\Sigma_s$ & 97 & 47 & 71 & 56 & 59  & \textbf{330} \\
            \bottomrule
        \end{tabular}
        \label{subtab:experiment_1}
    \end{subtable}
    \hfill %
    \begin{subtable}{0.52\textwidth}
        \centering
        \caption{\ppomlp{} Agents.}
        \begin{tabular}{lccccccc}
            \toprule
            &            & \multicolumn{5}{c}{\textbf{\ppomlp{}}} &                \\
            & seed       & 0    & 1    & 2   & 3     & 4    & $\Sigma_c$     \\
            \midrule
            \multirow{5}{*}{\rotatebox{90}{\textbf{\ppomlp{}}}}
            & 0          & ---  & 514  & 458  & 1690 & 539  & 3201           \\
            & 1          & 1486 & ---  & 1574 & 1765 & 1139 & 5964           \\
            & 2          & 1542 & 426  & ---  & 1882 & 874  & 4724           \\
            & 3          & 310  & 235  & 118  & ---  & 357  & 1020           \\
            & 4          & 1461 & 861  & 1126 & 1643 & ---  & 5091           \\
            & $\Sigma_s$ & 4799 & 2036 & 3276 & 6980 & 2909 & \textbf{20000} \\
            \bottomrule
        \end{tabular}
    \end{subtable}
    \\\vfill\vspace{0.5cm}
    \begin{subtable}{0.47\textwidth}
        \centering
        \caption{Absolute Coordinates vs. \relcoords{}.}
        \begin{tabular}{cccccccc}
            \toprule
            &          & \multicolumn{5}{c}{\textbf{OCT}} &     \\
            & seed     & 0  & 1  & 2  & 3 & 4     & $\Sigma_c$    \\
            \midrule
            \multirow{5}{*}{\rotatebox{90}{\textbf{ORBiT}}}
            & 0        & 2  & 1  & 1 & 4  & 1     & 9           \\
            & 1        & 1  & 1  & 0 & 5  & 4     & 11          \\
            & 2        & 0  & 2  & 0 & 1  & 1     & 4           \\
            & 3        & 17 & 9  & 3 & 10 & 4     & 43          \\
            & 4        & 0  & 0  & 0 & 5  & 1     & 6           \\
            & $\Sigma_s$ & 20 & 13 & 4 & 25 & 11    & \textbf{73} \\
            \bottomrule
        \end{tabular}
        \label{subtab:experiment_2}
    \end{subtable}
    \hfill %
    \begin{subtable}{0.52\textwidth}
        \centering
        \caption{OCT Agents.}
        \begin{tabular}{cccccccc}
            \toprule
            &            & \multicolumn{5}{c}{\textbf{OCT}} &                   \\
            & seed       & 0    & 1    & 2    & 3    & 4    & $\Sigma_c$        \\
            \midrule
            \multirow{5}{*}{\rotatebox{90}{\textbf{OCT}}}
            & 0          & ---  & 1052 & 1129 & 1174 & 999  & 4354              \\
            & 1          & 948  & ---  & 897  & 1002 & 904  & 3751              \\
            & 2          & 871  & 1103 & ---  & 1162 & 958  & 4094              \\
            & 3          & 826  & 998  & 838  & ---  & 782  & 3444              \\
            & 4          & 1001 & 1096 & 1042 & 1218 & ---  & 4357              \\
            & $\Sigma_s$ & 3646 & 4249 & 3906 & 4556 & 3643 & \textbf{20000}    \\
            \bottomrule
        \end{tabular}
    \end{subtable}
    \\\vfill\vspace{0.5cm}
    \begin{subtable}{0.47\textwidth}
        \centering
        \caption{OCT vs. \ppomlp{}.}
        \begin{tabular}{cccccccc}
            \toprule
            &          & \multicolumn{5}{c}{\textbf{OCT}} &                \\
            & seed       & 0   & 1   & 2    & 3   & 4     & $\Sigma_c$     \\
            \midrule
            \multirow{5}{*}{\rotatebox{90}{\textbf{\ppomlp{}}}}
            & 0          & 13  & 190 & 158  & 177  & 34   & 572           \\
            & 1          & 419 & 193 & 297  & 200  & 285  & 1394          \\
            & 2          & 12  & 59  & 73   & 66   & 16   & 226           \\
            & 3          & 14  & 10  & 9    & 11   & 13   & 57            \\
            & 4          & 189 & 478 & 500  & 355  & 124  & 1646          \\
            & $\Sigma_s$ & 647 & 930 & 1037 & 809  & 472  & \textbf{3895} \\
            \bottomrule
        \end{tabular}
    \end{subtable}
    \hfill %
    \begin{subtable}{0.52\textwidth}
        \centering
        \caption{ORBiT Agents.}
        \begin{tabular}{cccccccc}
            \toprule
            &          & \multicolumn{5}{c}{\textbf{ORBiT}} &             \\
            & seed     & 0    & 1    & 2    & 3    & 4    & $\Sigma_c$       \\
            \midrule
            \multirow{5}{*}{\rotatebox{90}{\textbf{ORBiT}}}
            & 0        & ---  & 1228 & 1379 & 750  & 963  & 4320           \\
            & 1        & 772  & ---  & 828  & 748  & 826  & 3174          \\
            & 2        & 621  & 1172 & ---  & 606  & 692  & 3091           \\
            & 3        & 1250 & 1252 & 1394 & ---  & 1311 & 5207           \\
            & 4        & 1037 & 1174 & 1308 & 689  & ---  & 4208           \\
            & $\Sigma_s$ & 3680 & 4826 & 4909 & 2793 & 3792 & \textbf{20000} \\
            \bottomrule
        \end{tabular}
    \end{subtable}
    
\end{table}

{\clearpage
\setlength{\tabcolsep}{0.68mm}
\renewcommand{\arraystretch}{1.2}
\newcommand{\longtabcaption}{\textbf{Per-game results on the Atari environments}}
\setlength{\aboverulesep}{0pt}
\setlength{\belowrulesep}{0pt}
\setlength{\extrarowheight}{.75ex}
\begin{longtable}[c]{lccccc} %
	\caption{\longtabcaption{}; mean and standard deviation over 5 random seeds. Higher is better. Games marked with $^\ast$ do not have a \variant{translate} variant as the game screen content covers the entire frame and a lossless translation is not possible.}
	\label{tab:ext_atari}                                                                                                                                                      \\
	\toprule %
	\multirow{2}{*}{Game}                                    & \multicolumn{2}{c}{\relcoords{}}                     & \multicolumn{2}{c}{Absolute Coordinates} & Pixel         \\
                                                             \cmidrule(lr){2-3}                            \cmidrule(lr){4-5}        \cmidrule(lr){6-6}                        %
                                                             & ORBiT                 & PPObj                & OCT                  & PPObj                & \ppocnn{}          \\
	\midrule %
	\endfirsthead %
	\bottomrule %
	\endfoot %
	\caption{\longtabcaption{} (continued).}                                                                                                                                   \\
	\toprule %
	\multirow{2}{*}{Game}                                    & \multicolumn{2}{c}{\relcoords{}}                     & \multicolumn{2}{c}{Absolute Coordinates} & Pixel         \\
                                                             \cmidrule(lr){2-3}                            \cmidrule(lr){4-5}        \cmidrule(lr){6-6}                        %
                                                             & ORBiT                 & PPObj                & OCT                  & PPObj                & \ppocnn{}          \\
	\midrule %
	\endhead %
	\game{Alien}                                                    & $1511\pm727$          & $1658\pm530$         & $474\pm341$          & $748\pm605$          & $1644\pm649$       \\*
	\rowcolor{shade} \variant{translate}                     & $1686\pm782$          & $1613\pm545$         & $242\pm297$          & $584\pm323$          & $277\pm223$        \\*
	\rowcolor{shade} \variant{shuffle}                       & $1639\pm713$          & $310\pm186$          & $494\pm314$          & $228\pm89$           & ---                \\*
	\rowcolor{shade} \variant{last\_egg}                     & $1584\pm703$          & $1741\pm603$         & $100\pm303$          & $60\pm240$           & $1908\pm1275$      \\*
	\midrule %
	\game{Amidar}                                                   & $195\pm88$            & $347\pm78$           & $105\pm78$           & $255\pm124$          & $578\pm128$        \\*
	\rowcolor{shade} \variant{translate}                     & $182\pm83$            & $305\pm92$           & $113\pm79$           & $91\pm105$           & $22\pm28$          \\*
	\rowcolor{shade} \variant{shuffle}                       & $193\pm95$            & $15\pm19$            & $94\pm62$            & $6\pm13$             & ---                \\*
	\rowcolor{shade} \variant{pig\_enemies}                  & $189\pm85$            & $343\pm70$           & $185\pm154$          & $131\pm53$           & $591\pm123$        \\*
	\midrule %
	\game{Asterix}                                                  & $213\,542\pm125\,728$ & $1265\pm566$         & $10848\pm12061$      & $1226\pm612$         & $2800\pm503$       \\*
	\rowcolor{shade} \variant{translate}                     & $187\,259\pm127\,913$ & $1002\pm599$         & $3388\pm5786$        & $739\pm505$          & $256\pm364$        \\*
	\rowcolor{shade} \variant{shuffle}                       & $234\,747\pm129\,292$ & $197\pm170$          & $10487\pm8155$       & $237\pm134$          & ---                \\*
	\midrule %
	\game{Boxing}                                                   & $90\pm8$              & $91\pm7$             & $-11\pm25$           & $93\pm8$             & $92\pm6$           \\*
	\rowcolor{shade} \variant{translate}                     & $88\pm8$              & $92\pm7$             & $-30\pm30$           & $18\pm48$            & $11\pm34$          \\*
	\rowcolor{shade} \variant{shuffle}                       & $88\pm8$              & $32\pm35$            & $-11\pm26$           & $10\pm12$            & ---                \\*
	\rowcolor{shade} \variant{switch\_position}              & $92\pm7$              & $90\pm7$             & $-11\pm25$           & $93\pm7$             & $91\pm6$           \\*
	\midrule %
	\game{Breakout}                                                 & $330\pm99$            & $193\pm109$          & $323\pm108$          & $131\pm87$           & $205\pm105$        \\*
	\rowcolor{shade} \variant{translate}                     & $309\pm113$           & $183\pm118$          & $98\pm139$           & $95\pm78$            & $9\pm18$           \\*
	\rowcolor{shade} \variant{shuffle}                       & $314\pm111$           & $2\pm2$              & $313\pm99$           & $2\pm1$              & ---                \\*
	\rowcolor{shade} \variant{gravity}                       & $277\pm120$           & $213\pm117$          & $273\pm114$          & $155\pm98$           & $185\pm101$        \\*
	\midrule %
	\game{Freeway}$^\ast$                                 & $32\pm1$              & $32\pm1$             & $31\pm1$             & $31\pm1$             & $32\pm1$           \\*
	\rowcolor{shade} \variant{shuffle}                       & $32\pm1$              & $12\pm7$             & $31\pm1$             & $16\pm3$             & ---                \\*
	\rowcolor{shade} \variant{stop\_all\_cars\_tunnel}       & $32\pm1$              & $32\pm1$             & $8\pm16$             & $0\pm0$              & $32\pm1$           \\*
	\midrule %
	\game{Frostbite}$^\ast$                               & $261\pm14$            & $288\pm126$          & $258\pm14$           & $516\pm638$          & $285\pm19$         \\*
	\rowcolor{shade} \variant{shuffle}                       & $258\pm13$            & $51\pm51$            & $256\pm15$           & $66\pm38$            & ---         \\*
	\rowcolor{shade} \variant{reposition\_floes\_easy}       & $215\pm45$            & $168\pm46$           & $118\pm110$          & $32\pm44$            & $46\pm70$          \\*
	\midrule %
	\game{Gopher}                                                   & $7985\pm13564$        & $6808\pm9599$        & $2866\pm5012$        & $2611\pm3088$        & $1149\pm716$       \\*
	\rowcolor{shade} \variant{translate}                     & $8046\pm14037$        & $8614\pm11327$       & $2188\pm3111$        & $964\pm1186$         & $782\pm545$        \\*
	\rowcolor{shade} \variant{shuffle}                       & $13776\pm21489$       & $438\pm403$          & $2753\pm3784$        & $415\pm218$          & ---                \\
	\game{Kangaroo}$^\ast$                                & $3526\pm1433$         & $3756\pm4485$        & $4946\pm4064$        & $4344\pm3904$        & $6188\pm3895$      \\*
	\rowcolor{shade} \variant{shuffle}                       & $3506\pm1416$         & $8\pm40$             & $5006\pm4124$        & $45\pm115$           & ---                \\*
	\rowcolor{shade} \variant{randomize\_kangaroo\_position} & $3406\pm1390$         & $3182\pm3615$        & $4828\pm3961$        & $4318\pm3676$        & $6006\pm3929$      \\*
	\midrule %
	\game{MsPacman}                                                 & $2449\pm1067$         & $3643\pm1210$        & $217\pm307$          & $2399\pm1024$        & $3907\pm723$       \\*
	\rowcolor{shade} \variant{translate}                     & $2464\pm1100$         & $3475\pm1220$        & $117\pm76$           & $1856\pm1003$        & $474\pm669$        \\*
	\rowcolor{shade} \variant{shuffle}                       & $2143\pm947$          & $396\pm167$          & $226\pm328$          & $536\pm586$          & ---                \\*
	\rowcolor{shade} \variant{set\_level\_2}                 & $214\pm183$           & $241\pm145$          & $413\pm369$          & $128\pm40$           & $466\pm591$        \\*
	\midrule %
	\game{Pong}                                                     & $13\pm7$              & $18\pm2$             & $18\pm2$             & $18\pm2$             & $17\pm4$           \\*
	\rowcolor{shade} \variant{translate}                     & $14\pm6$              & $18\pm3$             & $-15\pm11$           & $-9\pm13$            & $-19\pm7$          \\*
	\rowcolor{shade} \variant{shuffle}                       & $14\pm6$              & $-21\pm1$            & $18\pm3$             & $-21\pm1$            & ---                \\*
	\rowcolor{shade} \variant{hidden\_enemy}                 & $14\pm6$              & $18\pm2$             & $-20\pm2$            & $-19\pm4$            & $17\pm3$           \\*
	\midrule %
	\game{SpaceInvaders}                                            & $513\pm123$           & $488\pm163$          & $424\pm210$          & $367\pm174$          & $735\pm193$        \\*
	\rowcolor{shade} \variant{translate}                     & $552\pm130$           & $478\pm132$          & $334\pm165$          & $355\pm180$          & $258\pm195$        \\*
	\rowcolor{shade} \variant{shuffle}                       & $516\pm114$           & $147\pm94$           & $400\pm173$          & $218\pm134$          & ---                \\*
	\rowcolor{shade} \variant{disable\_shield\_left}         & $530\pm133$           & $335\pm157$          & $388\pm176$          & $384\pm192$          & $622\pm190$        \\*
	\midrule %
	\game{Tennis}                                                   & $18\pm4$              & $22\pm1$             & $12\pm10$            & $8\pm13$             & $-6\pm9$           \\*
	\rowcolor{shade} \variant{translate}                     & $16\pm6$              & $22\pm2$             & $-8\pm11$            & $-5\pm10$            & $-8\pm10$          \\*
	\rowcolor{shade} \variant{shuffle}                       & $18\pm4$              & $-23\pm2$            & $13\pm11$            & $-16\pm10$           & ---                \\*
\end{longtable}}
\FloatBarrier

\begin{table}[ht]
    \centering
    \caption{\textbf{The higher average score of OCT in \game{Kangaroo} depends on a single seed.} The other seeds result in very similar performance, yielding a comparable IQM as well. IQM over all episodes.}
    \begin{tabular}{lccccccc}
        \toprule
         \multirow{2}{*}{Method} & \multicolumn{5}{c}{Seed} & \multirow{2}{*}{IQM}\\
         \cmidrule{2-6}
         &  0 &  1 &  2 &  3 &  4 \\
        \midrule
        OCT   & 3980   & 1800   & 12680  & 4470   & 1800 & \textbf{3435} \\
        ORBiT & 4700   & 1800   & 4620   & 4710   & 1800 & \textbf{3727} \\
        \bottomrule
        \end{tabular}
    \label{tab:kangaroo}
\end{table}

\clearpage

\begin{table}[t]
  \centering
  \caption{\textbf{Per-configuration results in \game{Tag}}: catches IQM  with 95\% CI over 5 seeds, each averaged over 100 episodes. Higher is better. Extended results for \cref{fig:tag}. PPObj cannot be evaluated beyond its $N_{\max}$ object slots, i.e., 1 Player, 1 Target, 2 Landmarks, and 5 Distractors. Landmarks are always spaced evenly along the box diagonal and in orientation, so changing their number moves all of them. The number of distractors and landmarks is changed only during testing, without retraining.}
  \label{tab:tag-results}
  \begin{tabular}{lcccc}
    \toprule
    & \multicolumn{2}{c}{\relcoords{}} & \multicolumn{2}{c}{Absolute Coordinates} \\
    \cmidrule(lr){2-3} \cmidrule(l){4-5}
    Configuration & ORBiT & PPObj & OCT & PPObj \\
    \midrule
    Base ($m = 5$, Anchor: Player)    & 22.4\,{\scriptsize[21.9, 22.9]} & 21.6\,{\scriptsize[18.1, 22.0]} & 1.8\,{\scriptsize[1.2, 2.9]} & 5.6\,{\scriptsize[1.6, 7.4]} \\
    \rowcolor{shade}$m' = 0$ distractors              & 21.9\,{\scriptsize[21.1, 22.4]} & --- & 2.4\,{\scriptsize[1.0, 4.8]} & --- \\
    \rowcolor{shade}$m' = 1$ distractors              & 22.1\,{\scriptsize[21.7, 22.5]} & --- & 2.1\,{\scriptsize[1.0, 4.0]} & --- \\
    \rowcolor{shade}$m' = 10$ distractors             & 22.5\,{\scriptsize[21.8, 23.3]} & --- & 1.5\,{\scriptsize[1.1, 2.4]} & --- \\
    \rowcolor{shade}$k' = 0$ landmarks                      & 21.9\,{\scriptsize[20.7, 22.4]} & --- & 1.1\,{\scriptsize[0.8, 1.6]} & --- \\
    \rowcolor{shade}$k' = 4$ landmarks                       & 22.2\,{\scriptsize[21.8, 22.8]} & --- & 1.7\,{\scriptsize[1.1, 2.2]} & --- \\
    \rowcolor{shade}Anchor: Target                    & 20.7\,{\scriptsize[16.7, 21.6]} & 10.3\,{\scriptsize[0.9, 14.7]} & --- & --- \\
    \rowcolor{shade}Anchor: Distractor                & 0.7\,{\scriptsize[0.4, 0.8]} & 1.0\,{\scriptsize[0.5, 1.1]} & --- & --- \\
    \rowcolor{shade}Anchor: Landmark                  & 1.4\,{\scriptsize[0.7, 1.8]} & 5.4\,{\scriptsize[3.8, 12.5]} & --- & --- \\
    \bottomrule
  \end{tabular}
\end{table}

\section{Input Perturbations}
\label{app:perturbations}
In addition to testing in training environments, we investigate the generalization capabilities of \relcoords{} and ORBiT by evaluating the trained agents in perturbed environments. For that, we use a RAD-like perturbation, i.e., random \variant{translation}, random \variant{shuffling} the object set, and the \framework{HackAtari} framework. Examples are shown in \cref{fig:perturbations}, and full per-game results are in \cref{tab:ext_atari}.

\begin{figure}
    \centering
    \begin{subfigure}[t]{0.32\linewidth}
        \centering
        \includegraphics[height=0.66\linewidth]{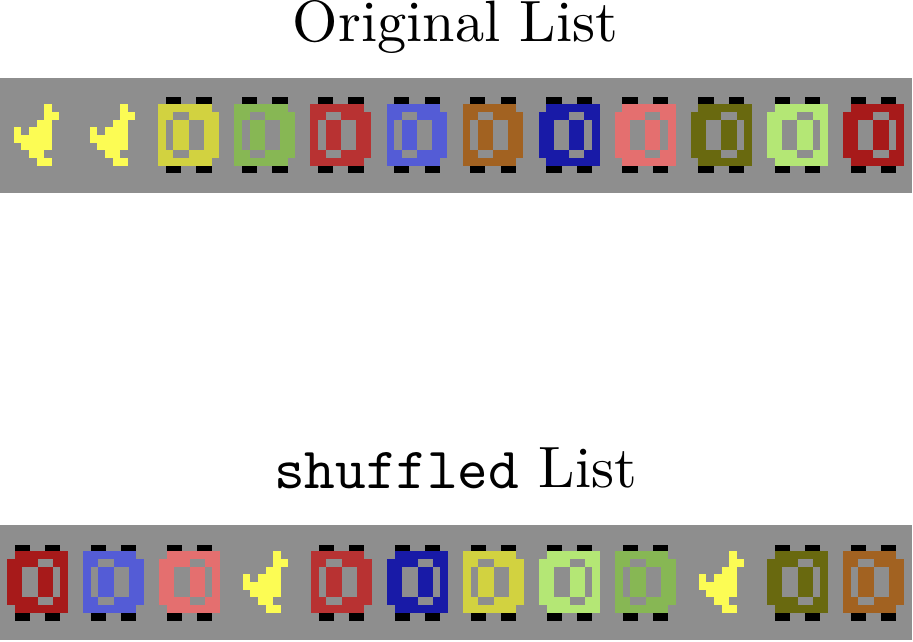}
        \caption{\variant{Shuffled} Object List}
    \end{subfigure}%
    \hfill
    \begin{subfigure}[t]{0.32\linewidth}
        \centering
        \includegraphics[height=0.66\linewidth]{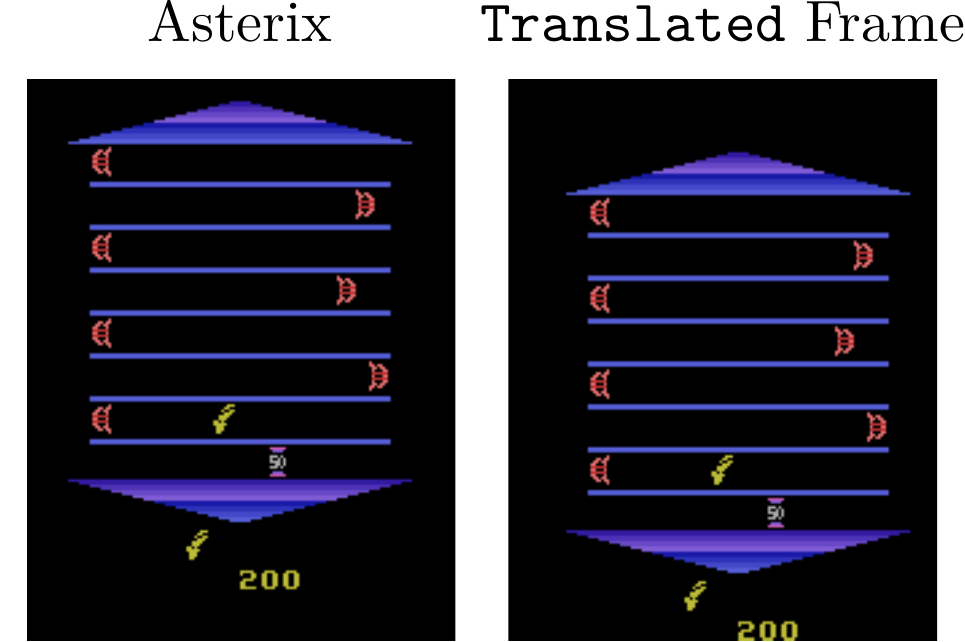}
        \caption{Random \variant{Translation}}
    \end{subfigure}%
    \hfill
    \begin{subfigure}[t]{0.32\linewidth}
        \centering
        \includegraphics[height=0.66\linewidth]{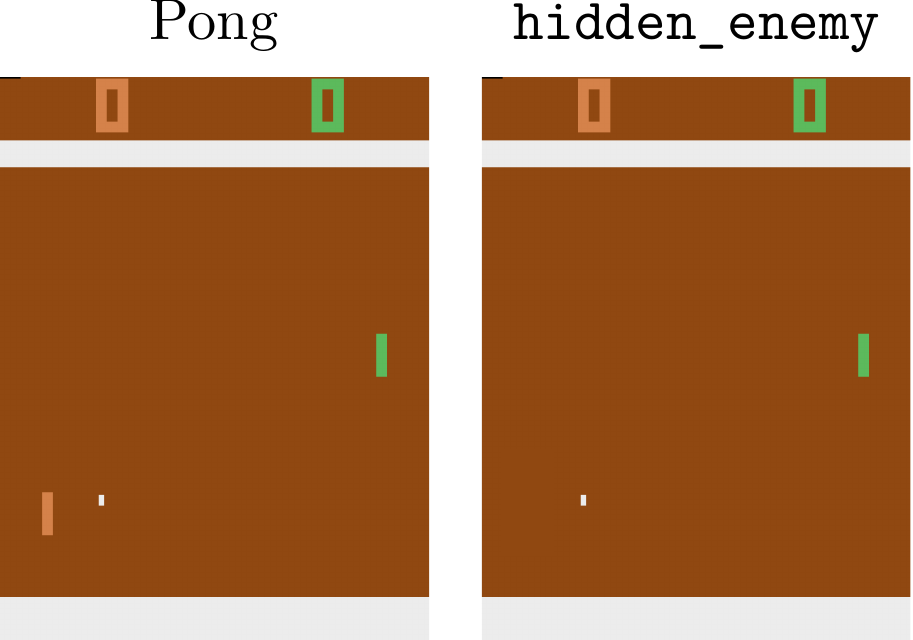}
        \caption{\framework{HackAtari}}
    \end{subfigure}%
    \caption{\textbf{The perturbed environment variants to test the generalizability of ORBiT and \relcoords{}.} In the \variant{shuffled} object list variant (a), the order of $\mathcal{O}_t$ (or $\tilde{\mathcal{O}}_t$) is randomized before it is passed to the agent. The random \variant{translation} (b) shifts the actual game content within the video frame without cutting any content. The \framework{HackAtari} framework (c) features a variety of variants of the Atari games from level changes to logic changes, e.g., hiding the opponent in \game{Pong}.}
    \label{fig:perturbations}
\end{figure}

\clearpage
\end{document}